\DocumentMetadata{pdfstandard=A-3u}

\PassOptionsToPackage{dvipsnames,svgnames,x11names}{xcolor}

\documentclass[
subscriptcorrection,
upint,
varvw,
barcolor=Goldenrod3,
mathalfa=cal=euler,
balance,
hyphenate,
french
]{asmejour}

\usepackage{amsthm}
\usepackage{textcomp}
\usepackage{float}
\usepackage{soul}
\usepackage{algorithmic}
\usepackage{algorithm}
\usepackage{comment}
\usepackage{tabularx}
\mathtoolsset{showonlyrefs}

\def\BibTeX{{\rm B\kern-.05em{\sc i\kern-.025em b}\kern-.08em
    T\kern-.1667em\lower.7ex\hbox{E}\kern-.125emX}}

\newtheorem*{definition*}{Definition}

\newtheorem{remark}{Remark}
\newtheorem{lemma}{Lemma}
\newtheorem{proposition}{Proposition}

\hypersetup{%
	pdfauthor={Ahmed Ali, Chiara Gabellieri, Antonio Franchi},                       		   	
	pdftitle={Lie Group Formulation of Recursive Dynamics Algorithms of Higher Order for Floating-Base Robots},                  	
	pdfkeywords={Floating-base systems, Lie group formulation, higher-order forward dynamics, higher-order hybrid dynamics, higher-order inverse dynamics, admissible Coriolis matrix},
	pdfsubject = {Robotics},			
	pdflicenseurl={https://ctan.org/pkg/asmejour},
}

\JourName{Mechanisms and Robotics}
\PreprintString{}
                   
\begin{document}

\SetAuthorBlock{Ahmed Ali\CorrespondingAuthor}{Robotics and Mechatronics Department,\
Electrical Engineering, Mathematics, and Computer Science (EEMCS) Faculty,\
University of Twente, 7500 AE Enschede, The Netherlands.\
Email: ahmed.ali@utwente.nl}

\SetAuthorBlock{Chiara Gabellieri}{Robotics and Mechatronics Department, EEMCS Faculty,\
University of Twente, 7500 AE Enschede, The Netherlands.\
Email: c.gabellieri@utwente.nl}

\SetAuthorBlock{Antonio Franchi}{Robotics and Mechatronics Department, EEMCS Faculty,\
University of Twente, 7500 AE Enschede, The Netherlands,\
and Department of Computer, Control and Management Engineering,\
Sapienza University of Rome, 00185 Rome, Italy.\
Email: schol@r.franchi.eu}

\title{Lie Group Formulation of Recursive Dynamics Algorithms of {{Higher Order}} for Floating-Base Robots}


\begin{abstract}
{{In this paper, {{we describe procedures for computing higher-order time derivatives of the Lie-group Newton-Euler, Articulated-Body Inertia, and hybrid dynamics}} algorithms for floating-base trees, where the base configuration evolves on $\mathrm{SE}(3)$ and the attached mechanism is an open kinematic tree with configuration on the $(n_1+n_2)$-dimensional manifold $\mathbb{T}^{n_1}\times \mathbb{R}^{n_2}$, using spatial representation of twists. After presenting the algorithms, we collect the resulting recursions into closed-form equations of motion, identifying an admissible Coriolis matrix satisfying the passivity property, and showing that the articulated inertia tensor remains unchanged across all time derivatives. We then apply the developed methods to a 12-DoF aerial manipulator to derive analytical expressions for its geometric forward and inverse dynamics along with their first time derivatives whereas the numerical simulations successfully evaluate these dynamics up to fifth order. {{Finally, to demonstrate their practical utility, we benchmark the proposed extensions and show that, in the considered tests, their computational cost scales quadratically with the derivative order, whereas the automatic-differentiation baseline exhibits exponential scaling.}}}}
\end{abstract}

\date{\today}

\maketitle 


\section{Introduction}
Rigid body dynamics has been a central area of study in the science of mechanics since the era of Galileo, and the field has now reached a mature stage of development \cite{bottema1990theoretical,siciliano2010robotics}. There are several approaches for deriving the dynamics models of articulated rigid body systems, with the most prominent paradigms being Lagrangian, Hamiltonian, and Newtonian frameworks. As robotic systems have grown increasingly complex, with designs characterized by having high degrees of freedom, different mechanisms of actuation and wider set of tasks they can perform, the demand for computationally efficient methods to model their dynamics has also surged \cite{angeles2003fundamentals}. Among these methods that are particularly suited for digital computation of dynamics is the Newtonian-based recursive approach. A comparison of the computational complexities of some of these methods applied to open chain serial kinematics is reported in \cite{comparision}.

Conventionally, this approach uses a 3D or 6D (spatial) vectorial representation of the kinematics to formulate the Newton-Euler equations of the dynamics between any two interconnected bodies, which can then be evaluated recursively \cite{STEPANENKO1976137,ORIN1979107}. This idea became the foundation behind a set of specialized recursive algorithms detailed in \cite{featherstone2014rigid,jain2010robot,lilly2002efficient,Rodriguez1,Rodriguez2} that deal with the problem of computing inverse and forward dynamics of mechanisms. While this approach has achieved a huge success, which can be seen from its implementation in widely-used software programs like MATLAB and ADAMS \cite{ComputerComparision,ADAMS}, it does not fully leverage the intrinsic geometric properties of the quantities involved, remaining essentially coordinate-dependent and thereby limiting its application to a local chart of the configuration manifold. Other non-geometric  approaches exist; an example is \cite{lee2005new}, where the authors presented a method exploiting a particular diagonal structure of the inverted mass matrix of serial manipulators to develop an efficient algorithm for its inversion, solving the forward dynamics problem for these manipulators.

Moreover, building on this classical the Newton-Euler framework, the authors of \cite{FI_NoDer_NOGeo} derive an algorithm for computing the dynamics of floating-base systems when a modified Denavit–Hartenberg (DH) description is used for the kinematics. A related approach appears in \cite{Snake_NOGeo}, which applies a classical Newton-Euler-based model to snake robots operating in constrained environments. Their approaches employ a minimal parameterization of the base. In contrast, the work in \cite{F_NODer_NoGeo_para} adopts a non-recursive approach to the forward dynamics problem, although the base orientation remains expressed using Euler angles. Similarly, \cite{Diff} presents an algorithm that yields the dynamic model through symbolic differentiation, using spatial algebra in the kinematic description. 

To address this limitation, research into geometric representations of rigid body mechanics has been carried out, enabling a global, coordinate-free formulation of dynamics. This line of research has produced a significant body of literature known as geometric mechanics. A seminal contribution in this direction is the influential paper by Brockett \cite{brockett2005robotic}, introducing the use of Lie algebra and screw theory in modeling the forward kinematics of manipulators through the concept of Product of Exponentials (PoE) as a geometric parallel to the classical Denavit–Hartenberg method. Lie groups and screw theory sit in the core of this geometric framework which is described in a number of textbooks \cite{davidson2004robots,murray2017mathematical,lynch2017modern}. In \cite{Liegroup,mechAnalysi}, an application of this approach is illustrated on a fixed-base closed-kinematics mechanism, and respectively, constructing an analysis for linkages' mobility.

In \cite{4542865}, singularity-free closed-form laws of motion are derived by formulating the Euler-Lagrange equations on global, non-Euclidean configuration space, while \cite{Inertially-Decoupled} extended the analysis further by providing a coordinate transformation by which the EoMs of articulated floating-base systems can be decoupled into two separate equations: one for the dynamics of whole system's CoM, i.e. the centroidal dynamics \cite{orin2013centroidal}, and one for the internal variables (joints). Employing Lie-theoretic tools for manipulators whose kinematics is represented by PoE, with the joint axes being screws, the method in \cite{park1995lie} calculates the corresponding inertia matrix, Coriolis terms, and gravitational torque. Founded on the finite element concept \cite{Geradin_Cardona_2001}, another geometric approach to obtain singularity-free closed-form EoM for multibody systems on $SE(3)$ is proposed in \cite{Oliver1,Olivier2}. It leads to a system of differential algebraic equations (DAEs), whose dimension is $6N_B+6N_J$ with $N_B$ and $N_J$ denoting the number of bodies and joints in the system, respectively, due to representing the kinematical connection between bodies in the form of holonomic constraints with the help of Lagrange multipliers. Although this method is well-suited to model complex kinematic constraints and flexibility in the interconnection, it produces higher-dimensional EoMs compared to methods yielding ordinary differential equations (ODEs) EoMs with minimal dimension, equal the system's DoFs, thanks to the use of the screw theory to describe kinematic pairs. Additional model reduction techniques are needed to decrease the system dimension. Hence, this geometric finite element method can be more computationally expensive than calculating closed-form EoMs from the geometric O(N) Newton-Euler method in \cite{park1995lie}. 

On the other hand, algorithms that recast the recursive Newton-Euler approach for manipulators into the geometric setting are presented in \cite{Coordinate-inva,park2018geometric,advances}. An overview of these algorithms for serial manipulators using hybrid, body-fixed and spatial reference frames is discussed in \cite{muller2018screw,muller2018screw2}. In \cite{FI_NoDer_ParaBase}, a compact inverse dynamic formulation is achieved thanks to the use of Lie group methods; however, the base remains parameterized, and the algorithm does not provide a recursive calculation for forward dynamics. Despite its extensive treatment of the manipulator case, this geometric recursive formulation does not explicitly consider articulated floating-base systems in which the base configuration is taken as the entire $SE(3)$ instead of some parameterization of it, except the work in \cite{ReducedEL} where it is hinted out that EoM of a floating-base system can be obtained by recursively computing each term in the closed-form Hamel’s equations of motion \cite{hamel2013theoretische} using the Newton-Euler algorithms \cite{park1995lie}.

Another representation of the base configuration can be quaternion. The work in \cite{I_NoDer_Quat} utilizes a geometric formulation for inverse dynamics, using dual quaternion representation for kinematics. A similar treatment is observed in \cite{I_Lagra_NoDer_Quat}, where the Lagrangian equations are casted in dual quaternions algebra. Despite its elegance, these methods inherit the issue of double covering in representing the group of spatial rotations $SO(3)$, which complicates the attitude modeling of free-floating bases.

{{None of these works derive analytical \emph{time} derivatives of inverse or forward dynamics for floating-base trees with the base modeled on $SE(3)$, despite many applications that require such derivatives. Although differentiation can be carried out via automatic, numerical, or analytical (closed-form or recursive) methods, analytical differentiation is often preferred for control design: e.g., feedback linearization of emerging multirotors \cite{ali2026pendumavsixinputomnidirectionalmav} or manipulators with elastic joints requires second-order inverse dynamics, obtained either by time differentiating EoMs \cite{Elastic,elasticSpong} or by computing second-order time derivatives of inverse dynamics via the classical Newton-Euler recursion \cite{Deluca}; similarly, many optimal trajectory planning problems \cite{time-optimal,Gasparetto2015} and optimal control \cite{optimalCont} evaluate time derivatives of inverse dynamics. Existing motion-planning works do not fill this gap: \cite{advances} provides no analytical higher-order \emph{time} derivatives; \cite{trodynamics} lacks an explicit ${SE}(3)$ floating base and reports partial derivatives w.r.t.\ joint variables (not time derivatives); \cite{optimalmotyionwilly} derives only first-order partial derivatives of a Lie-group Newton-Euler model w.r.t.\ path parameters, while the authors in \cite{hong2022geometric} treat a special $\mathrm{SE}(3)$ base case where all children attach directly to the base (or to a wing), thus excluding arbitrary open-tree topologies and time differentiation. Motivated by this, time derivatives of geometric inverse dynamics for manipulators are illustrated in \cite{Müller2021Closed-form,nthorder,spatialtwist,garofalo2013closed}.}}
\begin{table*}
\centering
\caption{Existing Methods for Dynamics Modeling}
\renewcommand{\arraystretch}{1.2}
\begin{tabular}{|p{3cm}|c|c|c|c|p{1cm}|p{1cm}|c|c|c|}
\hline
\textbf{Method} & \textbf{$SE(3)$} & \textbf{Geom.} & \textbf{Closed} & \textbf{Rec.} & \textbf{Geom. Der.} & \textbf{High Der.} & \textbf{O($N$)} & \textbf{Spatial} & \textbf{F/I} \\
\hline
Park et. al \cite{park2018geometric,park1995lie,Coordinate-inva}  & $\times$ & $\checkmark$ &  $\checkmark$ & $\checkmark$ & $\times$ & $\times$ & $\checkmark$ & $\times$ & F/I \\
Hollerbach et. al \cite{comparision} & $\times$ & $\times$ & $\times$ & $\checkmark$ & $\times$ & $\times$ & $\checkmark$ & $\times$ & I \\
Lilly, Orin \cite{lilly2002efficient,ORIN1979107} & $\times$ & $\times$ & $\times$ & $\checkmark$ & $\times$ & $\times$ & $\checkmark$ & $\times$ & F/I\\
Orin et al \cite{orin2013centroidal} & $\times$ & $\times$ & $\checkmark$ & $\checkmark$ & $\times$ & $\times$ & $\checkmark$ & $\times$ & I\\
Lee, Chirikjian \cite{lee2005new} & $\times$ & $\times$ & $\checkmark$ & $\times$ & $\times$ & $\times$ & $\checkmark$ & $\times$ & F\\
Featherstone \cite{featherstone2014rigid} & $\times$ & $\times$ & $\times$ & $\checkmark$ & $\times$ & $\times$ & $\checkmark$ & $\times$ & F/I\\
Rodriguez et. al \cite{Rodriguez1,Rodriguez2} & $\times$ & $\times$ & $\times$ & $\checkmark$ & $\times$ & $\times$ & $\checkmark$ & $\times$& F/I \\
Garofalo et. al \cite{garofalo2013closed} & $\times$ & $\checkmark$ & $\checkmark$ & $\checkmark$ & $\checkmark$ & $\times$ & $\checkmark$ & $\times$ & I\\
Garofalo et. al \cite{Inertially-Decoupled} & $\checkmark$ & $\checkmark$ & $\checkmark$ & $\times$ & $\times$ & $\times$ & - & $\times$ & I\\
Khalil et. al \cite{FI_NoDer_NOGeo} & $\times$ & $\times$ & $\times$ & $\checkmark$ & $\times$ & $\times$ & $\checkmark$ & $\times$ & F/I\\
Kumar et. al \cite{FI_NoDer_ParaBase} & $\times$ & $\checkmark$ & $\checkmark$ & $\checkmark$ & $\times$ & $\times$ & $\checkmark$ & $\checkmark$ & I\\
Afonso Silva et. al \cite{I_NoDer_Quat} & $\times$ & $\checkmark$ & $\times$ & $\checkmark$ & $\times$ & $\times$ & $\checkmark$ & $\times$ & I\\
Cohen et. al \cite{I_Lagra_NoDer_Quat} & $\times$ & $\checkmark$ & $\checkmark$ & $\times$ & $\times$ & $\times$ & - & $\times$ & I\\
Mishra et. al \cite{ReducedEL} & $\checkmark$ & $\checkmark$ & $\checkmark$ & $\times$ & $\times$ & $\times$ & - & $\times$ & I\\
O. Brüls et. al \cite{Oliver1} & $\checkmark$ & $\checkmark$ & $\checkmark$ & $\times$ & $\times$ & $\times$ & - & $\times$ & I\\
Müller \cite{muller2018screw2} & $\times$ & $\checkmark$ & $\checkmark$ & $\checkmark$ & $\times$ & $\times$ & $\checkmark$ & $\checkmark$ & I\\
Müller \cite{spatialtwist} & $\times$ & $\checkmark$ & $\times$ & $\checkmark$ & $\checkmark$ & $\times$ & $\checkmark$ & $\checkmark$ & I\\
Kumar, Müller \cite{nthorder} & $\times$ & $\checkmark$ & $\checkmark$ & $\checkmark$ & $\checkmark$ & $\checkmark$ & $\checkmark$ & $\times$& I \\
\textbf{Proposed method} & \textbf{$\checkmark$} & $\checkmark$ & $\checkmark$ & $\checkmark$ & $\checkmark$ & $\checkmark$ & $\checkmark$ & $\checkmark$ & F/I\\
\hline
\end{tabular}
\vspace{1mm}

\footnotesize{
\textbf{$SE(3)$}: Singularity-free Floating-base Representation; 
\textbf{Geom.}: Frame/coordinate Invariance of the EoM;  
\textbf{Closed}: Transformation to Closed-form Lagrangian EoM;
\textbf{Rec.}: Recursive Formulation;
\textbf{Geom. Der.}: Exact Geometric Time Derivatives; 
\textbf{High Der.}: higher-order time Derivatives; 
\textbf{O($N$)}: Linear Complexity; 
\textbf{Spatial}: Spatial frame representation;
\textbf{F/I}: Computes Forward/Inverse Dynamics.
}
\label{tab:method_comparison}
\end{table*}

{{This work fills the gap in \emph{recursive higher-order time differentiation of geometric inverse, forward and hybrid dynamics for floating-base systems with base configuration on $SE(3)$ and an attached open kinematic tree with prismatic and revolute joints (internal variables in $\mathbb{T}^{n_1}\times\mathbb{R}^{n_2}$), using spatial representation of twists}. Without loss of generality, the base may be actuated by a total propeller wrench applied at its CoM. In particular, the paper makes the following contributions:
\begin{itemize}
\item Algorithm \ref{alg:Dyn} (H-GRNE-FBS): Extending the {{existing}} Recursive Newton-Euler formulation of \cite{Coordinate-inva,spatialtwist} from fixed-base manipulators to compute higher-order time derivatives of geometric inverse dynamics for the aforementioned class of floating-base systems.
\item Algorithm \ref{alg:ForwardDyn} (H-GABI-FBS): Extending the {{existing}} articulated-body inertia algorithm of Featherstone \cite{featherstone2014rigid,park2018geometric} to compute higher-order time derivatives of geometric forward dynamics for the aforementioned class of floating-base systems.
\item Algorithm \ref{alg:HybridDyn} (H-GHYB-FBS): Extending the {{existing}} hybrid dynamics algorithm \cite{featherstone2014rigid,park2018geometric} to compute higher-order time derivatives of geometric hybrid dynamics for the aforementioned class of floating-base systems.
\item Proposition \ref{eomprop}: Systematically collecting the resulting recursions into closed-form Lagrangian equations of motion and its time derivatives.
\item Lemma \ref{corolilemma}: Stating that an admissible Coriolis matrix satisfying the standard passivity property can be easily constructed from the recursions for this class.
\item Lemma \ref{timedreivative_inv}: Providing a structural insight into higher-order geometric forward dynamics: the articulated body inertia \cite{featherstone2014rigid} remains unchanged across higher-order time derivatives of the dynamics, facilitating efficient time differentiation since this inertia can be computed once and transported to any order.
\item Extensive simulation and benchmarking, including comparison against the automatic differentiation method implemented in CasADi \cite{Andersson2018}, {{providing evidence of the efficiency of the proposed extensions and their favorable performance over automatic differentiation for computing higher-order time derivatives of the geometrically-exact dynamics in large tree-structured systems.}}
\end{itemize}}}

{{The advantages of this higher order formulation are inherited from those well-established in the Lie-group literature at the acceleration order (0th order), namely coordinate-invariant expressions, singularity-free, global and unique representation of $SE(3)$ base, and compact recursions~\cite{park2018geometric} even for higher order time derivatives (see Table \ref{tab:recursive_all}). At the same time, the resulting computations remain highly computationally tractable as the derivative order increases, as shown by the simulations in Sec.~7.}} Table \ref{tab:method_comparison} offers a non-exhaustive list contrasting some well-known methods against our proposed extensions.

The rest of the paper is organized as follows. Section~\ref{sec:preli} gives a brief summary of fundamental concepts used throughout the paper. Afterwards, the two passes of the proposed Higher Order Geometric Recursive Newton-Euler algorithm for Floating Base Systems (H-GRNE-FBS) are presented in Section~\ref{inverse} while the relation between such algorithm and the closed-form EoM is explained in Section~\ref{relation}. We devote Section~\ref{forward} for the introduction of the second algorithm: Higher Order Geometric Recursive Articulated Body Inertia algorithm for Floating Base Systems (H-GABI-FBS). {{Hybrid dynamics and their time derivatives (H-GHYB-FBS) are the subject of Section \ref{hybridsec}. An application of these algorithms to an aerial manipulator is explained in Section~\ref{example} with simulations and benchmarking}}. We conclude by some remarks on implementation and future work in Section~\ref{conc}.

\section{Preliminaries}
\label{sec:preli}
In this section we briefly recall the definitions of some basic concepts from Lie group and screw theory literature which are used extensively throughout the paper. The unfamiliar reader is encouraged to consult  \cite{murray2017mathematical,291887,bottema1990theoretical,brockett2005robotic,lynch2017modern,muller2018screw,muller2018screw2,park1995lie,featherstone2014rigid,sola2018micro} and references therein for a more formal introduction. Free rigid body motion in the 3D Euclidean space has a configuration manifold that can be characterized by the set of homogeneous transformation matrices, that also form a matrix Lie group, called the special Euclidean group, denoted by $SE(3)$ and defined as
\begin{equation}
    \begin{split}
    SE(3)=\Big\{(\mathbf{R},\boldsymbol{x})=&\left(\begin{smallmatrix}
        \mathbf{R} & \boldsymbol{x} \\ \mathbf{0}_{1\times3} & 1
    \end{smallmatrix}\right)\in  \mathrm{GL}(4,\mathbb{R})\,|\\&\, \boldsymbol{x}\in \mathbb{R}^3,\,\mathbf{R}^T\mathbf{R}=\mathbf{R}\mathbf{R}^T=\mathbf{I}_{3\times3}, \, \det(\mathbf{R})=+1\Big\}
\end{split}
\end{equation}
where $\mathrm{GL}(4,\mathbb{R})$ is the general linear group of the set of real invertible matrices in $\mathbb{R}^{4\times4}$. The associated Lie algebra is a vector space tangent to the group identity and equipped with a Lie bracket as an extra structure. It is denoted by $se(3)$ and given as
\begin{equation}
    se(3)=\Big\{\left(\begin{smallmatrix}
        [\boldsymbol{\omega}] & \boldsymbol{v} \\ \mathbf{0}_{1\times3} & 0
    \end{smallmatrix}\right)\in \mathbb{R}^{4\times4}\,|\,  [\boldsymbol{\omega}]+[\boldsymbol{\omega}]^T=\mathbf{0},(\boldsymbol{\omega},\boldsymbol{v})\in \mathbb{R}^6\Big\}
\end{equation}
where the operator $[\cdot]$ is a linear isomorphism from the vector space $\mathbb{R}^{p}$ into the corresponding Lie algebra. Since $se(3)$ is isomorphic to $\mathbb{R}^{6}$, $[\boldsymbol{g}]\in se(3)$ can take the vector form $\boldsymbol{g}=(\boldsymbol{\omega},\boldsymbol{v})^T \in \mathbb{R}^{6}$. The adjoint action of the Lie group $SE(3)$ on its Lie algebra $se(3)$ can be expressed in a matrix form, called the adjoint representation of the Lie group, which is modulated by the group element and denoted by $\boldsymbol{\mathrm{Ad}}_{\mathbf{C}}$. It is defined as 
\begin{equation}
\begin{split}
    \boldsymbol{\mathrm{Ad}}_{\mathbf{C}}:&\, SE(3)\times se(3) \rightarrow se(3),\,\boldsymbol{\mathrm{Ad}}_{\mathbf{C}}\,\\&\boldsymbol{g}=\mathbf{C}[\boldsymbol{g}]\mathbf{C}^{-1},\, [\boldsymbol{g}] \in se(3), \,\, \mathbf{C} \in SE(3)\\
    \boldsymbol{\mathrm{Ad}}_{\mathbf{C}}=&\left(\begin{matrix}
    \mathbf{R} & \mathbf{0} \\ [\boldsymbol{x}]\mathbf{R} & \mathbf{R}
    \end{matrix}\right), \text{ and for $SE(3)$ } \boldsymbol{\mathrm{Ad}}_{\mathbf{C}}^*:= \boldsymbol{\mathrm{Ad}}_{\mathbf{C}}^T
\end{split}
\end{equation}
where $\boldsymbol{\mathrm{Ad}}_{\mathbf{C}}^*$ is the co-adjoint matrix. The Lie bracketing of two $se(3)$ elements $\boldsymbol{g}_1$ and $\boldsymbol{g}_2$, denoted by $[\boldsymbol{g}_1,\boldsymbol{g}_2]$, introduces a non-commutative and non-associative skew-symmetric multiplication operation on the vector space, that is modulated by Lie algebra elements, as follows 
\begin{equation}
\begin{split}
    \boldsymbol{\mathrm{ad}}_{\boldsymbol{g}_1}:&\, se(3)\times se(3) \rightarrow se(3)\\
    [\boldsymbol{g}_1,\boldsymbol{g}_2]&=\boldsymbol{\mathrm{ad}}_{\boldsymbol{g}_1} \boldsymbol{g}_2, \quad \boldsymbol{g}_1,\boldsymbol{g}_2 \in se(3)  
\end{split}
\end{equation}
where the adjoint $\boldsymbol{\mathrm{ad}}_{\boldsymbol{g}}$ and the co-adjoint $\boldsymbol{\mathrm{ad}}_{\boldsymbol{g}}^*$ can then be put in this matrix form for the $SE(3)$ group
\begin{equation}
        \boldsymbol{\mathrm{ad}}_{\boldsymbol{g}}=\left(\begin{matrix}
    [\boldsymbol{\omega}] & \mathbf{0} \\ [\boldsymbol{v}] & [\boldsymbol{\omega}]
    \end{matrix}\right), \text{ and for $se(3)$ } \boldsymbol{\mathrm{ad}}_{\boldsymbol{g}}^*:= \boldsymbol{\mathrm{ad}}_{\boldsymbol{g}}^T
\end{equation}
If $\mathbf{C}(t)\in SE(3)$ changes over time $t$, then the time derivatives of the inverse of the adjoint map $\boldsymbol{\mathrm{Ad}}_{\mathbf{C}(t)}$ and the Lie algebra adjoint $\boldsymbol{\mathrm{ad}}_{\boldsymbol{V}}$ are related to the right-invariant (spatial) twist $\boldsymbol{V}^s \in se(3)$ by
\begin{equation}
\Dot{\boldsymbol{\mathrm{Ad}}}_{\mathbf{C}(t)}^{-1}=-\boldsymbol{\mathrm{Ad}}_{\mathbf{C}}^{-1}\boldsymbol{\mathrm{ad}}_{\boldsymbol{V}^s}, \text{ and } \Dot{\boldsymbol{\mathrm{ad}}}_{\boldsymbol{V}}= \boldsymbol{\mathrm{ad}}_{\Dot{\boldsymbol{V}}}
\label{Adjointderiv}
\end{equation}
The following properties of the adjoint maps $\boldsymbol{\mathrm{ad}}_{\boldsymbol{V}}$ and $\boldsymbol{\mathrm{Ad}}_{\mathbf{C}}$ are useful in the subsequent development
\begin{equation}
        \begin{split} 
           \boldsymbol{\mathrm{ad}}_{\boldsymbol{V}_1+\boldsymbol{V}_2}&= \boldsymbol{\mathrm{ad}}_{\boldsymbol{V}_1}+ \boldsymbol{\mathrm{ad}}_{\boldsymbol{V}_2},\\
           \boldsymbol{\mathrm{ad}}_{r\boldsymbol{V}}&= r\,\boldsymbol{\mathrm{ad}}_{\boldsymbol{V}}, \, \, r\in \mathbb{R}, \text{and } \boldsymbol{\mathrm{ad}}_{\boldsymbol{V}} {\boldsymbol{V}}=0, \\
           \boldsymbol{\mathrm{Ad}}_{\mathbf{C}_1 \mathbf{C}_2}&=\boldsymbol{\mathrm{Ad}}_{\mathbf{C}_1}\,\boldsymbol{\mathrm{Ad}}_{\mathbf{C}_2},\\
           \boldsymbol{\mathrm{Ad}}_{\mathbf{C}}^{-1}&=\boldsymbol{\mathrm{Ad}}_{\mathbf{C}^{-1}},\\
           \boldsymbol{\mathrm{Ad}}_{\mathbf{C}}[\boldsymbol{g}_1,\boldsymbol{g}_2]&=[\boldsymbol{\mathrm{Ad}}_{\mathbf{C}}\boldsymbol{g}_1,\boldsymbol{\mathrm{Ad}}_{\mathbf{C}}\boldsymbol{g}_2], \text{ \quad(Lie group homomorphism)}
           \\
           \boldsymbol{\mathrm{ad}}_{[\boldsymbol{g}_1,\boldsymbol{g}_2]}&=[\boldsymbol{\mathrm{ad}}_{\boldsymbol{g}_1},\boldsymbol{\mathrm{ad}}_{\boldsymbol{g}_2}] \text{ \quad(Lie algebra homomorphism)}
           \label{eqFormulas}
    \end{split}
\end{equation}
Mozzi–Chasles' theorem states that any rigid body motion can be fully described by a translation along a line, called the screw axis, followed by a rotation around an axis parallel to that line. This composes a screw displacement. Screws are elements $\in se(3)$ that are naturally suited to encode the kinematical constraints between two interconnected rigid bodies imposed by the connecting joints. In this work we express the screw coordinates relative to the inertial frame. The spatial representation of a joint screw is measured in the reference (home) configuration of the mechanism, i.e., when all joint variables are zeros, and expressed in the inertial frame \cite{Higherderivatives}. This spatial screw $[\boldsymbol{Y}] \in se(3)$ is constant and given by
\begin{equation} 
    \begin{split}
           \boldsymbol{Y}&= (\boldsymbol{e}, [\boldsymbol{p}]\boldsymbol{e})^T\in \mathbb{R}^6 \text{ for a revolute joint with parameter  $q \in \mathbb{S}^1$},\\
      \boldsymbol{Y}&= (\mathbf{0}, \boldsymbol{e})^T\in \mathbb{R}^6 \text{ for a prismatic joint with parameter $q \in \mathbb{R}$}
      \label{constantScrew}
    \end{split}
\end{equation}
where $\boldsymbol{e}$ and $\boldsymbol{p}$ are a unit axis in the direction of the joint's motion and the position of any point on that axis, respectively, both expressed in the inertial frame and measured when the robot is at home configuration.

A fundamental tool that we utilize later to express rigid body motions in terms of screws is the exponential map $\exp(\boldsymbol{g})$ defined as
\begin{equation}
     \exp(\boldsymbol{g}):\, \boldsymbol{g} \mapsto \mathbf{C}, \quad \boldsymbol{g} \in se(3),\, \mathbf{C} \in SE(3).
\end{equation}
This map can be computed in a closed form. It is a local diffeomorphism that associates elements from the Lie algebra $se(3)$ to corresponding elements in the Lie group $SE(3)$. For this group, this map is globally surjective and locally bijective in a neighborhood around the origin of $se(3)$.
\section{{{H-GRNE-FBS: Higher Order Inverse Dynamics}}}
\label{inverse}
\begin{figure}
    \centering
    \includegraphics[width=0.45\columnwidth]{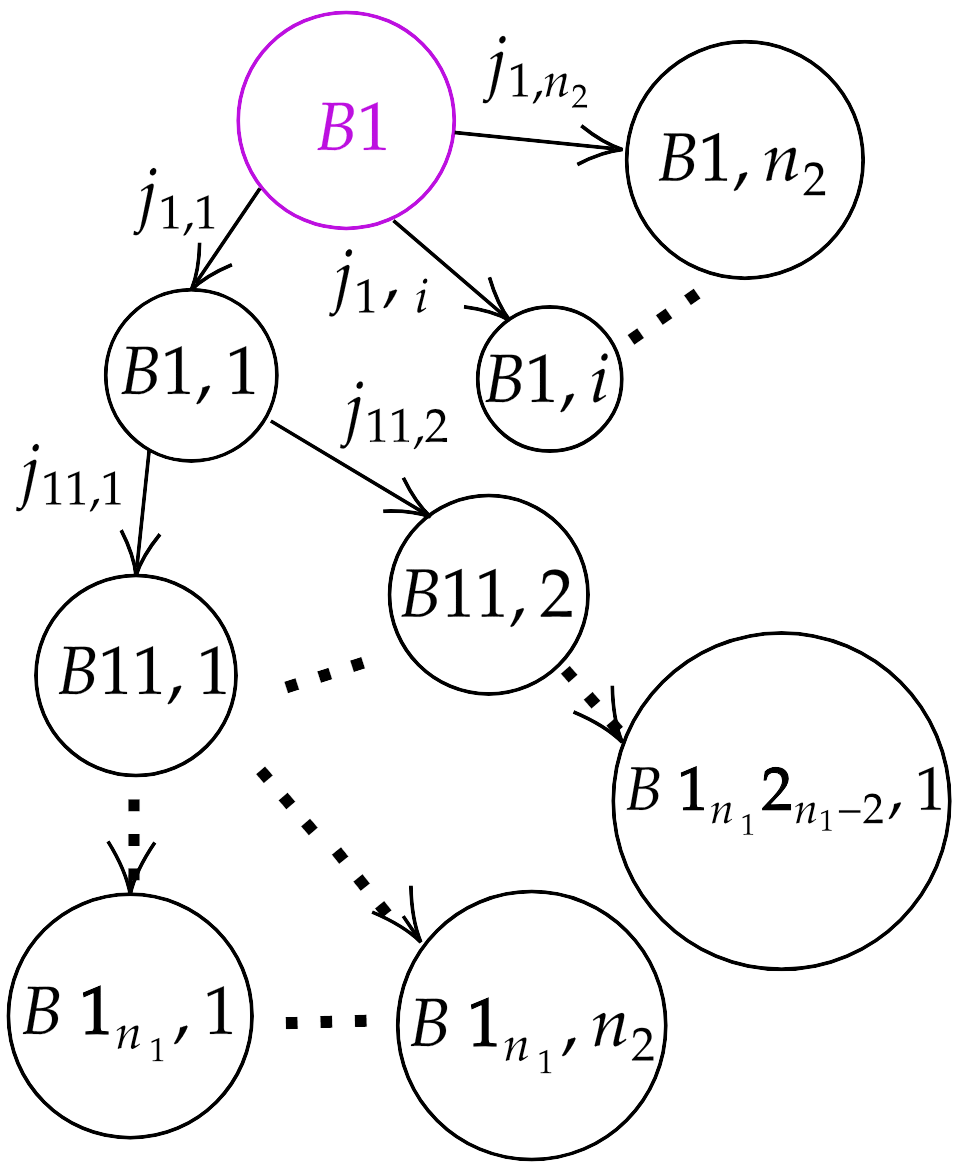}
    \caption{Directed graph with the nodes representing the rigid bodies while the edges are the joints $\in \mathbb{T}^{n_1}\times \mathbb{R}^{n_2}$, where $(n_1+n_2)$ is number of joints and B1 is the floating base. Notice that each body has only 1 parent.}
    \label{kinematicTree}
\end{figure}
We present {{higher order time derivatives of geometric Newton-Euler algorithm for inverse dynamics \cite{park1995lie} adapted to floating-base robots in a computationally-efficient and compact manner.}} This extends the $O(n)$ algorithm appeared in \cite{spatialtwist,Coordinate-inva} to yield the {{time derivatives of dynamics for trees with a floating base}}. In particular, we restrict our analysis to the case where the system has one main base evolving on $SE(3)$ that is connected to kinematic chains living in the $(n_1+n_2)$-dimensional manifold $\mathbb{T}^{n_1}\times \mathbb{R}^{n_2}$, with joint leading to body $j$ denoted by $q_j$, each of which has an open tree topology. In other words, as depicted with the directed graph in Fig.~\ref{kinematicTree}, the links attached to the floating-base form trees with no loop closure constraints (closed kinematical loops), meaning that each body has only one parent. 
The algorithm starts by a forward recursion to compute the {{higher order}} kinematics which is followed by a backward pass to obtain the {{corresponding order}} of inverse dynamics. {{Let $\mathbf{C}^{j}_{i}$, $\mathbf{A}^{j}_{i}$, $\mathbf{V}^0_i$, $(\boldsymbol{V}^{0}_{j})^{(i)}$ denote the relative configuration from frame $\mathcal{F}_i$ to $\mathcal{F}_j$, their relative constant home configuration, and spatial twist of $\mathcal{F}_i$ and its $i$-th component-wise time derivative, respectively. The total number of attached bodies is $N$ whereas the children and parent lists of $j$ are, respectively, $c\{j\}$ and $p\{j\}$. $(*)^{(r)}$ signifies $r$-th time derivative of $*$.}}

\subsection{Higher-order Forward Kinematics: Forward Recursion} 
The first recursion is traditionally called the forward recursion since the kinematics of each body on the tree depends on quantities that are deduced from its parent's parameters. Therefore, the kinematics can then be iteratively computed starting from the base body and passing once on each body by typically following a pre-order Depth-First Search (DFS) approach. We clarify {{initialization steps in Appendix \ref{append:initial}}} needed to run Algorithm (\ref{alg:kine}). Afterwards, the $r$th-order kinematics of body $i$, with the help of the product of exponential (PoE) \cite{brockett2005robotic} and the formulas \eqref{eqFormulas}, can be recursively computed by Algorithm (\ref{alg:kine}).

\begin{algorithm}
\caption{Higher-order Recursive Forward Kinematics for Floating-base Trees, extension of \cite{spatialtwist,Coordinate-inva}.}
\label{alg:kine}
\begin{algorithmic}
\STATE \textbf{Input:} $N,\, c\{j\},\, p\{j\}, \mathbf{A}^{0}_{1}, \mathbf{A}^{p\{j\}}_{j}, \mathbf{C}^{0}_{1}, \boldsymbol{V}^{0}_{1} \cdots (\boldsymbol{V}^{0}_{1})^{(r)}$,
\;$\boldsymbol{q} \cdots \dot{\boldsymbol{q}} \cdots  \boldsymbol{q}^{(r+1)}, \boldsymbol{Y}_1\,\, \boldsymbol{Y}_2 \cdots \boldsymbol{Y}_n$
\STATE \textbf{Call:} Forward($1$); \text{ \% pre-order DFS.}
\STATE \textbf{function} Forward($j$) 
\IF{$j>$1}
        \STATE \textbf{0-th order (position) kinematics:}
        \STATE $\mathbf{F}_{j} =\mathbf{F}_{p\{j\}} \exp([\boldsymbol{Y}_j]\,q_j),\text{    with $\mathbf{F}_1=\mathbf{C}^{0}_{1}$}$
        \STATE $\mathbf{A}^{0}_{j} = \mathbf{A}^{0}_{p\{j\}} \mathbf{A}^{p\{j\}}_j$ \text{    \% Computed once and used throughout.}
        \STATE $\mathbf{C}^{0}_{j} =  \mathbf{F}_j \mathbf{A}^{0}_j$
        \STATE $\boldsymbol{S}_j = \boldsymbol{\mathrm{Ad}}_{ {\mathbf{F}_j}}\, \boldsymbol{Y}_j$
        \STATE \textbf{1-st order (velocity) kinematics:}
        \STATE $\boldsymbol{V}^{0}_{j} = \boldsymbol{V}^{0}_{p\{j\}} + \boldsymbol{S}_j\,\dot{q}_j$
        \STATE $\dot{\boldsymbol{S}}_j = \boldsymbol{\mathrm{ad}}_{\boldsymbol{V}^{0}_{j}}\,{\boldsymbol{S}}_j$
        \STATE \textbf{2-nd order (acceleration) kinematics:}
        \STATE $\dot{\boldsymbol{V}}^{0}_{j} = \dot{\boldsymbol{V}}^{0}_{p\{j\}} + \boldsymbol{S}_j\,\ddot{q}_j + \dot{\boldsymbol{S}}_j\,\dot{q}_j$
        \STATE $\Ddot{\boldsymbol{S}}_j = \boldsymbol{\mathrm{ad}}_{\dot{\boldsymbol{V}}^{0}_{j}}\,{\boldsymbol{S}}_j+\boldsymbol{\mathrm{ad}}_{{\boldsymbol{V}}^{0}_{j}}\,\dot{\boldsymbol{S}}_j$
        \STATE \textbf{General $(r_1-1)$-th order kinematics}
        \STATE $\boldsymbol{S}_j^{(r_1)} = \sum_{r_2=0}^{r_1-1} \binom{r_1-1}{r_2} \boldsymbol{\mathrm{ad}}_{(\boldsymbol{V}^{0}_{j})^{(r_2)}} \boldsymbol{S}_j^{(r_1-1-r_2)}$
        \STATE $(\boldsymbol{V}^{0}_{j})^{(r_1)} = (\boldsymbol{V}^{0}_{p\{j\}})^{(r_1)} + \sum_{r_2=0}^{r_1} \binom{r_1}{r_2} \boldsymbol{S}_j^{(r_2)} q^{(r_1-r_2+1)}$
        \STATE \text{with the binomial coefficient} $\binom{r_1}{r_2} = \frac{r_1!}{r_2!(r_1-r_2)!}$
    \ENDIF
\FOR{$\forall i=1 : \text{length}(c\{$j$\})$} 
    \STATE Forward\big(c\{$j$\}($i$)\big) \% i-th child of body j
\ENDFOR
    \STATE \textbf{end function}
\end{algorithmic}
\end{algorithm}

\subsection{Higher-order Inverse Dynamics: Backward Recursion}
The second recursion starts from the last node in each branch of the tree and proceeds backward until it ends at the base, following a post-order DFS or a reversed breadth-first traversal. This is because the wrenches are transferred through the joints to the base. This recursion computing the higher-order geometric inverse dynamics can be implemented by Algorithm (\ref{alg:Dyn}). For the ease of exposition, explicit derivations are moved to Appendix \ref{append:Derivations}.  {{Let $\boldsymbol{W}^0_{j,\mathrm{app}}$, $\boldsymbol{W}^0_{1,\mathrm{prop}}$, $\boldsymbol{W}^0_{j,\mathrm{grav}}$, and $\boldsymbol{W}^0_{j}$ be the sum of applied external wrenches at body $j$, propeller wrench at the base 1, the gravitational wrench applied at body $j$, the constraint wrench that is transmitted through the joint $j$ from body $j$ to its parent (equals zero for the  $SE(3)$ floating base), respectively. $\tau_{j,\mathrm{ext}}$ is total non-conservative external torques on joint $j$ whereas $\tau_j$ is its actuation torque if present.}}

\begin{algorithm}
\caption{H-GRNE-FBS: Higher-order Geometric Recursive Inverse Dynamics for Floating-base Systems, extension of \cite{spatialtwist,Coordinate-inva}.}
\label{alg:Dyn}
\begin{algorithmic}
\STATE \textbf{Input:} Kinematics Algorithm  \ref{alg:kine}, $(\tau_{j,\mathrm{ext}})^{(r)}$ and $(\boldsymbol{W}^0_{j,\mathrm{app}})^{(r)}$.
\STATE \textbf{Output:} $\,(\boldsymbol{Q}_{j})^{(r)}$ \% generalized forces.
\STATE \textbf{Call:} Backward($1$); \text{    \% post-order DFS.}
\STATE \textbf{function} Backward($j$) 
    \FOR{$\forall i=1 : \text{length}(c\{$j$\})$} 
    \STATE Backward\big(c\{$j$\}($i$)\big) \text{    \% i-th child of body j}
    \ENDFOR
                \STATE $(\boldsymbol{W}^0_{j})^{(r)} =(\boldsymbol{\Pi}^0_{j})^{(r+1)}\eqref{Momu}- (\boldsymbol{W}^0_{j,\mathrm{app}})^{(r)}\eqref{extweriv} - (\boldsymbol{W}^0_{j,\mathrm{grav}})^{(r)} \eqref{gravweriv}+\sum_{i\in c\{j\}}(\boldsymbol{W}^0_{i})^{(r)}$
                \STATE
                \IF{$j=1$}
                \STATE {{$(\boldsymbol{Q}_{1})^{(r)}:=(\boldsymbol{W}^0_{1})^{(r)},$}} 
                \STATE{{ \text{hence}\,\,$(\boldsymbol{W}^0_{1,\mathrm{prop}})^{(r)}=(\boldsymbol{Q}_{1})^{(r)}$}}
                \ELSE
\STATE {{$(\boldsymbol{Q}_{j})^{(r)}:= \sum_{k=0}^{r} \binom{r}{k} (\boldsymbol{S}_{j}^T)^{(r-k)} (\boldsymbol{W}^0_{j})^{(k)}$,\, }}\STATE {{\text{hence} $(\tau_j)^{(r)}=(\boldsymbol{Q}_{j})^{(r)}-(\tau_{j,\mathrm{ext}})^{(r)}$}}
            \ENDIF
\STATE \textbf{end function}
\STATE
\end{algorithmic}
\end{algorithm}
\section{Relation to Lagrangian Equations of Motion}
\label{relation}
Here we give an example of how to put the dynamics equations obtained from Algorithm (\ref{alg:Dyn}) into the Lagrangian form of EoM. {{In Appendix \ref{prop1_proof}}}, we illustrate the derivations for the two cases {{laid out in Proposition \ref{eomprop}}}: the 0-th and 1-th order inverse dynamics. The closed-form EoM for higher orders follow the same procedure. Moreover, an admissible Coriolis matrix satisfying the passivity property is given {{Lemma \ref{corolilemma}}}. Following \cite{Coordinate-inva}, for the Algorithms \ref{alg:kine} and \ref{alg:Dyn} we can put all recursions in a matrix form. First, construct the two boolean matrices $\mathbf{G}_p$ and $\mathbf{G}_c$ such that the parents matrix $\mathbf{G}_p$ contains the identity matrix $\mathbf{I}_{6 \times 6}$ in all parent indexes of body $i$ to the base 1 and the zero matrix $\mathbf{0}_{6 \times 6}$ otherwise, while the children matrix $\mathbf{G}_c$ contains the identity matrix $\mathbf{I}_{6 \times 6}$ in all child indexes of body $i$ and the zero matrix otherwise. Both are obtained from the structure of the tree and are transpose of each other for this type of graphs.
\begin{table}
\centering
\caption{{{Definitions for Grouping Recursions}}}
\label{tab:defi}
\normalsize
\setlength{\tabcolsep}{3pt}
\renewcommand{\arraystretch}{1.15}
\begin{tabular}{@{}l@{\;}c@{\;}l@{}}
\toprule
$\mathbf{G}_c$ & $:=$ & $\mathbf{G}_p^T$\\
$\boldsymbol{V}$ & $:=$ & $(\,\boldsymbol{V}^0_1,\,\boldsymbol{V}^0_2,\,\boldsymbol{V}^0_3,\cdots,\,\boldsymbol{V}^0_N\,)^T$\\
$\boldsymbol{\mathrm{ad}}_{\boldsymbol{V}}$ & $:=$ & $\operatorname{blockdiag}(\,\boldsymbol{\mathrm{ad}}_{\boldsymbol{V}^0_1},\,\boldsymbol{\mathrm{ad}}_{\boldsymbol{V}^0_2},\,\boldsymbol{\mathrm{ad}}_{\boldsymbol{V}^0_3},\cdots,\,\boldsymbol{\mathrm{ad}}_{\boldsymbol{V}^0_N}\,)$\\
$\boldsymbol{\mathrm{Ad}}_{\mathbf{C}}$ & $:=$ & $\operatorname{blockdiag}(\,\boldsymbol{\mathrm{Ad}}_{\mathbf{C}^0_1},\,\boldsymbol{\mathrm{Ad}}_{\mathbf{C}^0_2},\,\boldsymbol{\mathrm{Ad}}_{\mathbf{C}^0_3},\cdots,\,\boldsymbol{\mathrm{Ad}}_{\mathbf{C}^0_N}\,)$\\
$\dot{\boldsymbol{V}}$ & $:=$ & $(\,\dot{\boldsymbol{V}}^0_1,\,\dot{\boldsymbol{V}}^0_2,\,\dot{\boldsymbol{V}}^0_3,\cdots,\,\dot{\boldsymbol{V}}^0_N\,)^T$\\
$\mathbf{S}$ & $:=$ & $\operatorname{blockdiag}(\mathbf{I}_{6\times6},\boldsymbol{S}_2,\boldsymbol{S}_3,\cdots,\boldsymbol{S}_N)$\\
$\dot{\mathbf{S}}$ & $:=$ & $\operatorname{blockdiag}(\mathbf{0}_{6\times6},\dot{\boldsymbol{S}}_2,\dot{\boldsymbol{S}}_3,\cdots,\dot{\boldsymbol{S}}_N)$\\
$\mathbf{M}$ & $:=$ & $\operatorname{blockdiag}(\mathbf{M}^0_1,\mathbf{M}^0_2,\mathbf{M}^0_3,\cdots,\mathbf{M}^0_N)$\\
$\boldsymbol{W}_{I,\mathrm{app}}^0$ & $:=$ & $(\,\boldsymbol{W}^0_{1,\mathrm{app}},\,\boldsymbol{W}^0_{2,\mathrm{app}},\,\boldsymbol{W}^0_{3,\mathrm{app}},\cdots,\,\boldsymbol{W}^0_{N,\mathrm{app}}\,)^T$\\
$\boldsymbol{W}_{I,\mathrm{app}}^b$ & $:=$ & $(\,\boldsymbol{W}^b_{1,\mathrm{app}},\,\boldsymbol{W}^b_{2,\mathrm{app}},\,\boldsymbol{W}^b_{3,\mathrm{app}},\cdots,\,\boldsymbol{W}^b_{N,\mathrm{app}}\,)^T$\\
$\boldsymbol{W}_{I,\mathrm{grav}}^0$ & $:=$ & $(\,\boldsymbol{W}^0_{1,\mathrm{grav}},\,\boldsymbol{W}^0_{2,\mathrm{grav}},\,\boldsymbol{W}^0_{3,\mathrm{grav}},\cdots,\,\boldsymbol{W}^0_{N,\mathrm{grav}}\,)^T$\\
$\boldsymbol{G}^0_I$ & $:=$ & $(\,\boldsymbol{G}^0,\,\boldsymbol{G}^0,\,\boldsymbol{G}^0,\cdots,\,\boldsymbol{G}^0\,)^T\in\mathbb{R}^{6N\times1}$\\
$\dot{\boldsymbol{\Pi}}_{I}$ & $:=$ & $(\,\dot{\boldsymbol{\Pi}}^0_{1},\,\dot{\boldsymbol{\Pi}}^0_{2},\,\dot{\boldsymbol{\Pi}}^0_{3},\cdots,\,\dot{\boldsymbol{\Pi}}^0_{N}\,)^T$\\
$\Bar{\boldsymbol{\tau}}$ & $:=$ & $(\,\mathbf{0}_{6\times1},\,\tau_2+\tau_{2,\mathrm{ext}},\,\tau_3+\tau_{3,\mathrm{ext}},\cdots,\,\tau_N+\tau_{N,\mathrm{ext}}\,)^T$\\
\midrule
${{\boldsymbol{\mathcal{Q}}}}$ & $:=$ & $\begin{bmatrix}\mathbf{C}_1^0& \boldsymbol{q}_2& \boldsymbol{q}_3& \cdots& \boldsymbol{q}_N\end{bmatrix}$\\
${{\boldsymbol{\nu}^{(k-1)}}}$ & $:=$ & $\begin{bmatrix}(\boldsymbol{V}_1^0)^{(k-1)}& \boldsymbol{q}_2^{(k)}&\boldsymbol{q}_3^{(k)}& \cdots& \boldsymbol{q}_N^{(k)}\end{bmatrix}$ $k\geq 1$\\
\bottomrule
\end{tabular}
\end{table}
Here $\boldsymbol{W}^0_{1,\mathrm{prop}}$ is included in $\boldsymbol{W}^0_{1,\mathrm{app}}$ for brevity. Otherwise, the first entry in $\Bar{\boldsymbol{\tau}}$ is replaced with $\boldsymbol{W}^0_{1,\mathrm{prop}}$. {{ In the following, these definitions in Table \ref{tab:defi} hold. Now, we give a result in Proposition \ref{eomprop} on transforming the higher order recursions to closed-form EoM for the considered class of floating base robots.}}

\begin{proposition}[{{Transformation into closed-form EoM}}]
\label{eomprop}
{{
 The $0$-th and $1$-st order geometric inverse dynamics computed from recursions of (Algorithm~\ref{alg:Dyn}) can be written in closed form as follows:

\begin{itemize}
    \item The 0-th order inverse dynamics
\begin{equation}
        \Bar{\boldsymbol{\tau}} =\Bar{\mathbf{M}}({{\boldsymbol{\mathcal{Q}}}}) {{\dot{\boldsymbol{\nu}}}}+\boldsymbol{h}({{\boldsymbol{\mathcal{Q}}}}, {{{\boldsymbol{\nu}}}})+\boldsymbol{g}({{\boldsymbol{\mathcal{Q}}}}) + \boldsymbol{\tau}_{\mathrm{ext}}({{\boldsymbol{\mathcal{Q}}}}).
\label{EoM}
\end{equation}
With the mass matrix $\Bar{\mathbf{M}}$, Coriolis and centrifugal forces $\boldsymbol{h}$, gravitational vector $\boldsymbol{g}$, and external torque $\boldsymbol{\tau}_{\mathrm{ext}}$ defined as
\begin{align}
     \Bar{\mathbf{M}}({{\boldsymbol{\mathcal{Q}}}})&=  \mathbf{S}^T \mathbf{G}_c \mathbf{M}  \mathbf{G}_p\, \mathbf{S}, \\
     \boldsymbol{h}({{\boldsymbol{\mathcal{Q}}}},{{\boldsymbol{\nu}}})&= \mathbf{S}^T \mathbf{G}_c \mathbf{M} \mathbf{G}_p \dot{\mathbf{S}}\, {{\boldsymbol{\nu}}} - \mathbf{S}^T \mathbf{G}_c \boldsymbol{\mathrm{ad}}^T_{\mathbf{G}_p \mathbf{S} {{\boldsymbol{\nu}}}} \mathbf{M} \mathbf{G}_p\, \mathbf{S}\, {{\boldsymbol{\nu}}}, \nonumber \\
     \boldsymbol{g}({{\boldsymbol{\mathcal{Q}}}})&= - \mathbf{S}^T \mathbf{G}_c \,\boldsymbol{W}_{I,\mathrm{grav}}^0=- \mathbf{S}^T \mathbf{G}_c \mathbf{M} \boldsymbol{G}^0_I, \\
     \boldsymbol{\tau}_{\mathrm{ext}}({{\boldsymbol{\mathcal{Q}}}})&= - \mathbf{S}^T \mathbf{G}_c \,\boldsymbol{W}_{I,\mathrm{app}}^0=  - \mathbf{S}^T \mathbf{G}_c \boldsymbol{\mathrm{Ad}}^T_{\mathbf{C}^{-1}} \boldsymbol{W}_{I,\mathrm{app}}^b. \label{eqsprop}
\end{align} 
\item The 1st order inverse dynamics
\begin{equation}
       \Bar{\mathbf{M}}({{\boldsymbol{\mathcal{Q}}}}){{\ddot{\boldsymbol{\nu}}}}+\dot{\Bar{\boldsymbol{h}}}({{\boldsymbol{\mathcal{Q}}}},{{\boldsymbol{\nu}}}, {{\dot{\boldsymbol{\nu}}}})+\dot{\boldsymbol{g}}({{\boldsymbol{\mathcal{Q}}}},{{\boldsymbol{\nu}}})+\dot{\boldsymbol{\tau}}_{\mathrm{ext}}({{\boldsymbol{\mathcal{Q}}}},{{\boldsymbol{\nu}}})=\dot{\Bar{\boldsymbol{\tau}}}.
   \label{1_st_EoM}
\end{equation}
where
\begin{align}
     \Bar{\mathbf{M}}({{\boldsymbol{\mathcal{Q}}}})&=  \mathbf{S}^T \mathbf{G}_c \mathbf{M}  \mathbf{G}_p \mathbf{S},\\
     \dot{\Bar{\boldsymbol{h}}}({{\boldsymbol{\mathcal{Q}}}},{{\boldsymbol{\nu}}}, {{\dot{\boldsymbol{\nu}}}})&= \mathbf{S}^T\boldsymbol{\mathrm{ad}}^T_{\mathbf{G}_p \mathbf{S} {{\boldsymbol{\nu}}}} \big(\mathbf{G}_c \mathbf{M}  \mathbf{G}_p \mathbf{S} {{\dot{\boldsymbol{\nu}}}} + \mathbf{G}_c \mathbf{M} \mathbf{G}_p \boldsymbol{\mathrm{ad}}_{\mathbf{G}_p \mathbf{S} {{\boldsymbol{\nu}}}}  \mathbf{S} {{\boldsymbol{\nu}}} \nonumber \\
        & - \mathbf{G}_c \boldsymbol{\mathrm{ad}}^T_{\mathbf{G}_p \mathbf{S} {{\boldsymbol{\nu}}}} \mathbf{M} \mathbf{G}_p \mathbf{S} {{\boldsymbol{\nu}}} \big)+ \mathbf{S}^T\Tilde{\Ddot{\boldsymbol{\Pi}}}({{\boldsymbol{\mathcal{Q}}}},{{\boldsymbol{\nu}}},{{\dot{\boldsymbol{\nu}}}}),\\
     \dot{\boldsymbol{g}}({{\boldsymbol{\mathcal{Q}}}},{{\boldsymbol{\nu}}})&= - \mathbf{S}^T\boldsymbol{\mathrm{ad}}^T_{\mathbf{G}_p \mathbf{S} {{\boldsymbol{\nu}}}}  \mathbf{G}_c \mathbf{M} \boldsymbol{G}^0_I \\
        &+\mathbf{S}^T \mathbf{G}_c(\mathbf{M} \boldsymbol{\mathrm{ad}}_{\mathbf{G}_p \mathbf{S} {{\boldsymbol{\nu}}}}+\boldsymbol{\mathrm{ad}}^T_{\mathbf{G}_p \mathbf{S} {{\boldsymbol{\nu}}}} \mathbf{M}) \boldsymbol{G}^0_I,\\
     \dot{\boldsymbol{\tau}}_{\mathrm{ext}}({{\boldsymbol{\mathcal{Q}}}},{{\boldsymbol{\nu}}})&= \mathbf{S}^T(\mathbf{G}_c\boldsymbol{\mathrm{ad}}^T_{\mathbf{G}_p \mathbf{S} {{\boldsymbol{\nu}}}} \boldsymbol{\mathrm{Ad}}^T_{\mathbf{C}^{-1}} - \boldsymbol{\mathrm{ad}}^T_{\mathbf{G}_p \mathbf{S} {{\boldsymbol{\nu}}}} \mathbf{G}_c \boldsymbol{\mathrm{Ad}}^T_{\mathbf{C}^{-1}}) \boldsymbol{W}_{I,\mathrm{app}}^b \nonumber \\
     &-\mathbf{S}^T \mathbf{G}_c\boldsymbol{\mathrm{Ad}}^T_{\mathbf{C}^{-1}} \dot{\boldsymbol{W}}_{I,\mathrm{app}}^b. \label{eqprop1}
\end{align} 
\end{itemize}
}}
\end{proposition}
\begin{proof}
{{See Appendix \ref{prop1_proof}.}}
\end{proof}
{{\begin{remark}
The notation used in Table \ref{tab:defi} follows the right-invariant (spatial)
representation of the floating-base velocity $V_1^0\in se(3)$. Hence, the first component of the generalized velocity $\boldsymbol\nu$ is not the matrix derivative $\dot C_1^0\in T_{C_1^0}SE(3)$.
If tangent-space $r$-th derivative $(C_1^0)^{(r)}$ is needed, it can be retrieved using Leibniz' rule from
\begin{equation}
    (C_1^0)^{(r)}
=
\sum_{k=0}^{r-1}
\binom{r-1}{k}
[(V_1^0)^{(k)}](C_1^0)^{(r-1-k)},
\qquad r\geq 1
\end{equation}
\end{remark}}}
{{C\&C forces $\boldsymbol{h}$ in the 0-th order EoM can be compactly factorized with this admissible Coriolis matrix $\mathbf{C}$ given in the next Lemma \ref{corolilemma}, which can also be computed recursively.}}
\begin{lemma}[{{Admissible Coriolis Matrix}}]
\label{corolilemma}
    This matrix $\mathbf{C}$ 
   {{\begin{equation}
    \mathbf{C}({{\boldsymbol{\mathcal{Q}}}},{{\boldsymbol{\nu}}})= \mathbf{S}^T \mathbf{G}_c \mathbf{M} \mathbf{G}_p \dot{\mathbf{S}} - \mathbf{S}^T \mathbf{G}_c \boldsymbol{\mathrm{ad}}^T_{\mathbf{G}_p \mathbf{S} {{\boldsymbol{\nu}}}} \mathbf{M} \mathbf{G}_p \mathbf{S}.
    \label{Coroli}
\end{equation}}}
    satisfies the skew-symmetric property of the presented class of the Lagrangian systems $\tfrac{1}{2}\dot{\Bar{\mathbf{M}}}-\mathbf{C}$.
\end{lemma}

\begin{proof}
{{See Appendix \ref{append:proof_Lemma1}.}}
\end{proof}
\section{{{H-GABI-FBS: Higher Order Forward Dynamics}}}
\label{forward}
The forward dynamics and its time derivatives are not only required for simulation purposes, but also can be necessary for dynamic control schemes \cite{Deluca,ali2026pendumavsixinputomnidirectionalmav}. We can obtain these dynamics either from running the inverse dynamics algorithm and then solve a system of linear equations at each step or by using more computationally-efficient recursive algorithms such as composite inertia and articulated body inertia \cite{featherstone2014rigid}. In this section, we adapt the articulated-body algorithm (ABA) presented in \cite{featherstone2014rigid,Coordinate-inva} which is the fastest forward dynamics algorithm for kinematic trees with $N\geq6$, to compute higher order geometric forward dynamics for the aforementioned class of floating-base trees using Lie group theory and spatial representation of screws. This screw representation allows a more compact formulation when compared to its body-fixed representation counterpart \cite{Coordinate-inva} as it does not involve intermediate frame transformations of twists and wrenches.  
As with the inverse dynamics, the algorithm requires calculating the forward kinematics first. Then, it proceeds with a backward recursion in {{which the articulated inertia, which remains unchanged across orders as shown in Lemma \ref{timedreivative_inv} enabling efficient computation by calculating it once and carrying it over across orders}}, and bias wrench are evaluated for each body and propagated back to the floating base. 

{{
The derivatives of the wrenches, inertia tensors, $\boldsymbol{V}_{i,\mathrm{bias}}^0$, $\boldsymbol{\Pi}_{j,\mathrm{bias}}$, and the momentum map are calculated in Appendix \ref{append:Derivations}. Algorithm \ref{alg:ForwardDyn} provides the general procedure.}}

\begin{lemma}[{{Articulated-body inertia as an invariant leading coefficient under time differentiation}}]
\label{timedreivative_inv}
{{
For a rigid body $i$, denote the spatial articulated inertia by $\mathbf{M}_j^A$and the spatial bias wrench at order $r$ by $\boldsymbol{W}_{j,r}^A$. Then, for any $r\ge 0$, the $r$-th derivative of equations of motion are
\begin{equation}
(\boldsymbol{W}_j)^{(r)} = \mathbf{M}_j^A (\boldsymbol{V}_j^0)^{(r+1)} + \boldsymbol{W}_{j,r}^A,
\end{equation}
where $\mathbf{M}_j^A$, $\forall r\ge 0$, is given by 
\begin{equation}
     \mathbf{M}_{j}^A =  \mathbf{M}^0_{j} +\sum_{i\in c\{j\}} \mathbf{M}_{i}^A
        \!\left( \mathbf{I} - \boldsymbol{S}_{i} (\boldsymbol{S}_{i}^T \mathbf{M}_{i}^A \boldsymbol{S}_{i})^{-1} 
        \boldsymbol{S}_{i}^T \mathbf{M}_{i}^A \right),
\end{equation}
}}
\end{lemma}

\begin{proof}
    {{See Appendix \ref{append:proof_Lemma2}}}.
\end{proof}
\section{{{H-GHYB-FBS: Higher Order Hybrid Dynamics}}}
\label{hybridsec}
{{For completeness, Algorithm \ref{alg:HybridDyn} provides higher order time derivatives of Featherstone’s hybrid dynamics algorithm \cite{featherstone2014rigid} and adapt it for floating-base trees using Lie groups, where joints in $\mathcal{J}q$ have prescribed accelerations and higher derivatives while joints in $\mathcal{J}\tau$ have prescribed torques with their corresponding derivatives. For the $\mathrm{SE}(3)$ base, the propeller wrench can be given so that the algorithm yields the base acceleration twist and its derivatives, or the vice versa. This is useful, e.g., for trees with passive joints while the base tracks a desired twist trajectory, allowing computation of the propeller wrench and the evolution of the passive joints along that base trajectory.}}
\begin{algorithm}[!t]
\caption{H-GABI-FBS: Higher-Order Geometric Recursive Forward Dynamics for Floating-base Systems, extension of \cite{featherstone2014rigid,Coordinate-inva}.}
\label{alg:ForwardDyn}
\begin{minipage}{\columnwidth}
\small
\setlength{\abovedisplayskip}{1pt}
\setlength{\belowdisplayskip}{1pt}
\setlength{\abovedisplayshortskip}{1pt}
\setlength{\belowdisplayshortskip}{1pt}
\renewcommand{\algorithmicindent}{0.9em}
\begin{algorithmic}
\STATE \textbf{Input:} Alg. \ref{alg:kine}, 
$(\boldsymbol{W}^0_{1,\mathrm{prop}})^{(r)}$, $(\tau_{j})^{(r)}$,$(\tau_{j,\mathrm{ext}})^{(r)}$ $(\boldsymbol{W}^0_{j,\mathrm{app}})^{(r)}$.
\STATE \textbf{Output:} $\boldsymbol{q}^{(r+2)},(\boldsymbol{V}^0)^{(r+1)}$
\STATE \textbf{Call:} \texttt{Backward}(1); \hfill \% post-order DFS
\STATE \textbf{function} \texttt{Backward}($j$)
    \FOR{$\forall i=1 : \mathrm{length}(c\{j\})$}
        \STATE \texttt{Backward}\big(c\{j\}(i)\big) \hfill \% i-th child of body $j$
    \ENDFOR
    \STATE
    \begin{equation}
    \begin{split}
        \mathbf{M}_{j}^A &=  \mathbf{M}^0_{j} +\sum_{i\in c\{j\}} 
        \mathbf{M}_{i}^A\!\left( \mathbf{I} - \boldsymbol{S}_{i} 
        (\boldsymbol{S}_{i}^T \mathbf{M}_{i}^A \boldsymbol{S}_{i})^{-1} 
        \boldsymbol{S}_{i}^T \mathbf{M}_{i}^A \right),\\
        (\boldsymbol{W}_{j}^A)^{(r)} &=   
        -\boldsymbol{W}_{j,\mathrm{app}}^{(r)} - \boldsymbol{W}_{j,\mathrm{grav}}^{(r)} 
        + \boldsymbol{\Pi}_{j,\mathrm{bias}}^{(r+1)}
         +\sum_{i\in c\{j\}}\!
        \Big((\boldsymbol{W}_{i}^A)^{(r)}\\&+ 
        \mathbf{M}_{i}^A ( \boldsymbol{S}_{i} \Tilde{\boldsymbol{q}}^{(r+2)}_{i} 
        + (\boldsymbol{V}_{i,\mathrm{bias}}^0)^{(r+1)}\!\Big)
    \end{split}
    \label{eqMr}
    \end{equation}
    \STATE where $\Tilde{\boldsymbol{q}}^{(r+2)}_{i}$ and $\Tilde{\boldsymbol{\tau}}_{i}^{(r)}$ are given by
    \begin{equation}
    \begin{split}
        \Tilde{\boldsymbol{q}}^{(r+2)}_{i} &= 
        \!\left( \boldsymbol{S}_{i}^T \mathbf{M}^A_{i} \boldsymbol{S}_{i} \right)^{-1}
        \!\Big( {\boldsymbol{\tau}}_{i}^{(r)}+(\tau_{i,\mathrm{ext}})^{(r)}-\Tilde{\boldsymbol{\tau}}_{i}^{(r)}\\
        &\quad -\boldsymbol{S}_{i}^T 
        ( \mathbf{M}^A_{i}(\boldsymbol{V}_{i,\mathrm{bias}}^0)^{(r+1)}
        +(\boldsymbol{W}_{i}^A)^{(r)})\!\Big), \\
        &\text{ where $\Tilde{\boldsymbol{\tau}}_{i}^{(0)}=0$ otherwise given by}\\
        \Tilde{\boldsymbol{\tau}}_{i}^{(r)} &=
\sum_{k=0}^{r-1} \binom{r}{k}\, (\boldsymbol{S}_{i}^{T})^{(r-k)}
\left(\mathbf{M}_{i}^{A}(\boldsymbol{V}_{i}^{0})^{(k+1)} + (\boldsymbol{W}_{i}^{A})^{(k)}\right).
    \end{split}
    \label{dddqti}
    \end{equation}
\STATE \textbf{end function}
\vspace{0.5em}
\STATE \textbf{Call:} \texttt{Forward}(1); \hfill \% pre-order DFS
\STATE \textbf{function} \texttt{Forward}($j$)
\IF{$j=1$}
    \STATE Solve $(\boldsymbol{W}^0_{1,\mathrm{prop}})^{(r)} 
    = \mathbf{M}_{1}^A (\boldsymbol{V}_{1}^0)^{(r+1)} + (\boldsymbol{W}_{1}^A)^{(r)}$
    for $(\boldsymbol{V}_{1}^0)^{(r+1)}$ using $LDL^T$.
\ELSE
    \STATE
    \begin{equation}
    \begin{split}
        \boldsymbol{q}^{(r+2)}_{j} &= 
        -\!\left( \boldsymbol{S}_{j}^T \mathbf{M}^A_{j} \boldsymbol{S}_{j} \right)^{-1}
        \!\boldsymbol{S}_{j}^T \mathbf{M}^A_{j} (\boldsymbol{V}_{p\{j\}}^0)^{(r+1)}
        +\Tilde{\boldsymbol{q}}^{(r+2)}_{j},\\
        (\boldsymbol{V}_{j}^0)^{(r+1)} &= 
        (\boldsymbol{V}_{p\{j\}}^0)^{(r+1)}+ 
        \boldsymbol{S}_{j} \boldsymbol{q}^{(r+2)}_{j}
         +(\boldsymbol{V}_{j,\mathrm{bias}}^0)^{(r+1)}.
    \end{split}
    \end{equation}
\ENDIF
\FOR{$\forall i=1 : \mathrm{length}(c\{j\})$}
    \STATE \texttt{Forward}\big(c\{j\}(i)\big) \hfill \% i-th child of body $j$
\ENDFOR
\STATE \textbf{end function}
\end{algorithmic}
\end{minipage}
\end{algorithm}
\begin{algorithm}[!t]
\caption{\textcolor{black}{H-GHYB-FBS: Higher-Order Geometric Recursive Hybrid Dynamics for Floating-base Systems, extension of \cite{featherstone2014rigid}}}
\label{alg:HybridDyn}
\begin{minipage}{\columnwidth}
\small
\setlength{\abovedisplayskip}{1pt}
\setlength{\belowdisplayskip}{1pt}
\setlength{\abovedisplayshortskip}{1pt}
\setlength{\belowdisplayshortskip}{1pt}
\renewcommand{\algorithmicindent}{0.9em}
\begin{algorithmic}
\STATE \textbf{Input:} Alg.~\ref{alg:kine}; $\mathcal{J}_q,\mathcal{J}_\tau$, where $\mathcal{J}_q\cup\mathcal{J}_\tau=\{2,\dots,N\}$;
\STATE $\bar{\boldsymbol{q}}_{j}^{(r+2)}\ \forall j\in\mathcal{J}_q$;\quad $(\bar{\tau}_j)^{(r)}\ \forall j\in\mathcal{J}_\tau$;\quad $(\tau_{j,\mathrm{ext}})^{(r)}$; $(\boldsymbol{W}^0_{j,\mathrm{app}})^{(r)}$;
\STATE \text{either} $(\boldsymbol{W}^0_{1,\mathrm{prop}})^{(r)}$ \text{or} $(\boldsymbol{V}_1^0)^{(r+1)}$.
\STATE \textbf{Output:} $\boldsymbol{q}_{j}^{(r+2)}\ \forall j\in\mathcal{J}_\tau$;$(\boldsymbol{W}^0_{1,\mathrm{prop}})^{(r)}$ or $(\boldsymbol{V}_1^0)^{(r+1)}$;\ $(\tau_j)^{(r)}\ \forall j\in\mathcal{J}_q$;\ $(\boldsymbol{V}^0)^{(r+1)}$.
\STATE \textbf{Call:} Backward$(1)$; \hfill \% post-order DFS

\STATE \textbf{function} Backward$(j)$
\FOR{$\forall i=1:\mathrm{length}(c\{j\})$}
    \STATE Backward$\big(c\{j\}(i)\big)$
\ENDFOR
\STATE $\displaystyle
\mathbf{M}_{j}^A=\mathbf{M}_j^0
+\!\!\sum_{i\in c\{j\}\cap\mathcal{J}_q}\!\!\mathbf{M}_i^A
+\!\!\sum_{i\in c\{j\}\cap\mathcal{J}_\tau}\!\!
\mathbf{M}_{i}^A\!\left( \mathbf{I}-\boldsymbol{S}_{i}
(\boldsymbol{S}_{i}^T \mathbf{M}_{i}^A \boldsymbol{S}_{i})^{-1}
\boldsymbol{S}_{i}^T \mathbf{M}_{i}^A \right)$.
\STATE $\displaystyle
\begin{aligned}
(\boldsymbol{W}_{j}^A)^{(r)} &=
-(\boldsymbol{W}^0_{j,\mathrm{app}})^{(r)}-(\boldsymbol{W}^0_{j,\mathrm{grav}})^{(r)}
+\boldsymbol{\Pi}_{j,\mathrm{bias}}^{(r+1)}\\
&\quad+\sum_{i\in c\{j\}\cap\mathcal{J}_q}\Big((\boldsymbol{W}_{i}^A)^{(r)}+
\mathbf{M}_{i}^A(\boldsymbol{S}_{i}\bar{\boldsymbol{q}}_{i}^{(r+2)}
+(\boldsymbol{V}_{i,\mathrm{bias}}^0)^{(r+1)}\Big)\\
&\quad+\sum_{i\in c\{j\}\cap\mathcal{J}_\tau}\Big((\boldsymbol{W}_{i}^A)^{(r)}+
\mathbf{M}_{i}^A(\boldsymbol{S}_{i}\Tilde{\boldsymbol{q}}_{i}^{(r+2)}
+(\boldsymbol{V}_{i,\mathrm{bias}}^0)^{(r+1)})\Big).
\end{aligned}$
\STATE where $\Tilde{\boldsymbol{q}}_{i}^{(r+2)}$, $\Tilde{\boldsymbol{\tau}}_{i}^{(r)}$ are from \eqref{dddqti} with $(\tau_i)^{(r)}:=(\bar{\tau}_i)^{(r)}$ for $i\in\mathcal{J}_\tau$.
\STATE \textbf{end function}

\STATE \textbf{Call:} Forward$(1)$; \hfill \% pre-order DFS

\STATE \textbf{function} Forward$(j)$
\IF{$j=1$}
    \IF{$(\boldsymbol{W}^0_{1,\mathrm{prop}})^{(r)}$ is given}
        \STATE Solve $(\boldsymbol{W}^0_{1,\mathrm{prop}})^{(r)}=\mathbf{M}_1^A(\boldsymbol{V}_1^0)^{(r+1)}+(\boldsymbol{W}_1^A)^{(r)}$ for $(\boldsymbol{V}_1^0)^{(r+1)}$.
    \ELSE
        \STATE $(\boldsymbol{W}^0_{1,\mathrm{prop}})^{(r)}:=\mathbf{M}_1^A(\boldsymbol{V}_1^0)^{(r+1)}+(\boldsymbol{W}_1^A)^{(r)}$.
    \ENDIF
\ELSE
    \IF{$j\in\mathcal{J}_\tau$}
        \STATE $\displaystyle \boldsymbol{q}^{(r+2)}_{j}:=
        -(\boldsymbol{S}_{j}^T \mathbf{M}^A_{j} \boldsymbol{S}_{j})^{-1}
        \boldsymbol{S}_{j}^T \mathbf{M}^A_{j} (\boldsymbol{V}_{p\{j\}}^0)^{(r+1)}
        +\Tilde{\boldsymbol{q}}^{(r+2)}_{j}$.
    \ELSE
        \STATE $\boldsymbol{q}^{(r+2)}_{j}:=\bar{\boldsymbol{q}}^{(r+2)}_{j}$.
    \ENDIF
    \STATE $\displaystyle (\boldsymbol{V}_{j}^0)^{(r+1)} :=
    (\boldsymbol{V}_{p\{j\}}^0)^{(r+1)}+\boldsymbol{S}_{j}\boldsymbol{q}^{(r+2)}_{j}+(\boldsymbol{V}_{j,\mathrm{bias}}^0)^{(r+1)}$.
    \IF{$j\in\mathcal{J}_q$}
        \STATE $\displaystyle (\boldsymbol{Q}_{j})^{(r)}:= \sum_{k=0}^{r} \binom{r}{k} (\boldsymbol{S}_{j}^T)^{(r-k)} (\boldsymbol{W}^0_{j})^{(k)}$;\STATE hence\ 
        $(\tau_j)^{(r)}:=(\boldsymbol{Q}_{j})^{(r)}-(\tau_{j,\mathrm{ext}})^{(r)}$.
    \ENDIF
\ENDIF
\FOR{$\forall i=1:\mathrm{length}(c\{j\})$}
    \STATE Forward$\big(c\{j\}(i)\big)$
\ENDFOR
\STATE \textbf{end function}
\end{algorithmic}
\end{minipage}
\end{algorithm}
\section{{{Simulations and Benchmarking}}}
\label{example}

\begin{figure*}
\centering
\begin{subfigure}{1\columnwidth}
    \centering
    \includegraphics[width=1\columnwidth]{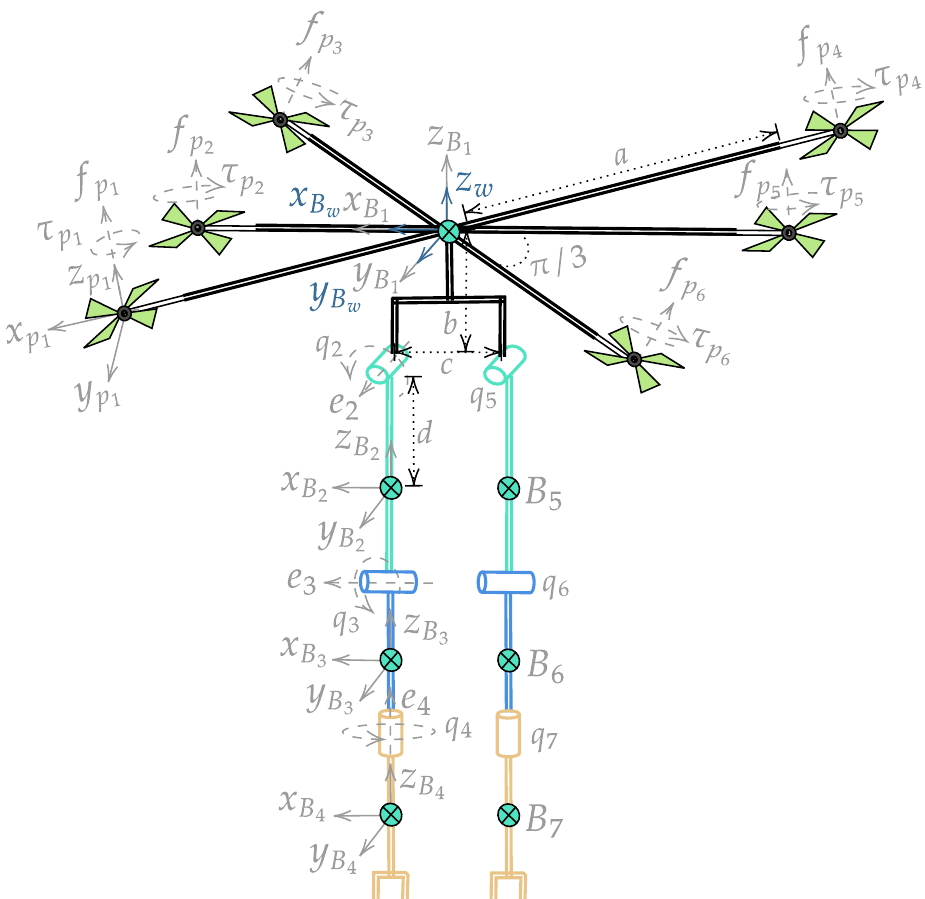}
    \caption{Fully-actuated base carrying 2 identical manipulators}
\end{subfigure}
\hspace{0.5cm}
\begin{subfigure}{0.5\columnwidth}
    \centering
    \includegraphics[width=1\columnwidth]{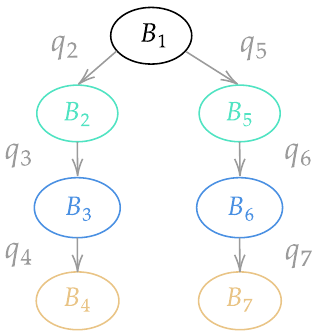}
    \caption{Graph of the system}
\end{subfigure}
\caption{a) Schematic representation of the floating base system in home configuration $\mathcal{H}_{0}$. The base is a fully-actuated hexarotor called TiltHex~\cite{TilThex2015}, and the attached manipulators have 3 revolute joints each, composing two 3-DoF flying end-effectors. The propeller $p_i$ is rigidly attached to the base and tilted around 2 axes. Frames $\mathcal{F}_{B_i}$ are at the configuration $\mathcal{H}_{0}$ and oriented identically to $\mathcal{F}_{w}$. b) Kinematic tree representing the system, with $B$ denoting a body (node) and $q$ a revolute joint (edge).}
\label{main_vehicle}
\end{figure*}
Aerial manipulation is typically enabled by an open kinematic chain mounted on a multirotor aerial vehicle \cite{aerialmanip,manipulationAerial}. In general, expressing their geometrically-exact dynamics and its time derivatives in a compact form that can also be efficiently computed in recursions is challenging due to the high DoFs and the presence of a free-floating base. The aerial manipulator in Fig.~\ref{main_vehicle}(a) provides an instance of the class of floating-base systems considered herein. 

{{In order to validate the proposed methods, we provide numerical simulations, whose results are depicted in Fig. \ref{fig:IDinout} and Fig. \ref{fig:IDinput}, of the higher-order inverse dynamics (H-GRNE-FBS), and forward dynamics algorithms (H-GABI-FBS) applied on this aerial manipulator robot, with the order ranging from the acceleration-level dynamics ($r=0$) up to its 5th time derivatives ($r=5$). We first describe the robot structure and the analytical application of the algorithms, summarized in Table~\ref{tab:recursive_all}. As an input to H-GRNE-FBS, we generate smooth reference trajectories for the floating-base configuration and joints, together with their derivatives, and plot them in Fig.~\ref{fig:IDinput}. Afterwards, the resulting higher-order base wrench and joint torques, shown in Fig.~\ref{fig:IDoutput}, are fed to H-GABI-FBS, which reproduces these reference trajectories and their derivatives at its output, as illustrated by the dashed lines in Fig.~\ref{fig:IDinput}. This confirms the successful operation of the algorithms. All simulation parameters are provided in Table~\ref{tab:tilthex_sim_params}. All computations are run on an Intel Core i7-12700H CPU (14 cores, up to 4.70~GHz).}} 

{{Moreover, we benchmark the computational cost of the proposed routines by measuring the mean time per call over 1000 repeated calls across multiple derivative orders $r$ and increasing number of bodies in the tree. This reports a quantitative indication of the practical complexity as a function of both the tree size and the derivative order. The benchmark setup settings are listed in Table \ref{tab:bench_params_short}. As shown in Fig. \ref{fig:IDbenchmark} and Fig. \ref{fig:FDbenchmark}, for a fixed order $r$, the algorithm computational time grows linearly in $N$, i.e. $O(N)$, while for a fixed $N$, they scale quadratically $O(r^2)$ in $r$ when tested on the tree of Fig. \ref{main_vehicle} but with the increased size of 5 branches and 20 bodies per branch. Although these results already highlight satisfactory scaling trends on par with other available packages such as Pinocchio \cite{carpentier2019pinocchio}, we note that further speed-ups can be possible through dedicated low-level implementations and code optimization techniques. Possible implementation is available at \url{https://github.com/ahmed11406/Dynamics/}.}}

{ Finally, we report a comparison against a baseline software: CasADi \cite{Andersson2018}, to quantify the cost of obtaining higher-order time derivatives via automatic differentiation of the 0-th order dynamic versus our proposed algorithms. Specifically, we implement the tree in Fig. \ref{main_vehicle} with the same increased size as used for benchmarking our proposed algorithms, symbolically in CasADi, and define the inverse-dynamics map $y_{\mathrm{ID}}(t)=\big[\boldsymbol{W}^0_{\text{prop},B_1}(t);\boldsymbol{\tau}(t)\big]$ and the forward-dynamics map $y_{\mathrm{FD}}(t)=\big[\dot{\boldsymbol{V}}^{0}_{B_1}(t);\ddot{\boldsymbol{q}}(t)\big]$. All CasADi evaluations are performed at $t_0=0$ with JIT compilation disabled. For each derivative order $r$, we generate the $r$-th time derivative $d^r y/dt^r$ once using CasADi’s built-in directional derivative operator \texttt{jtimes}, and we time only the subsequent repeated evaluations over 1000 calls, hence CasADi exponential time complexity displayed in Fig. \ref{fig:CasadAD} reflects the evaluation cost of an already-constructed $r$-th derivative expression graph without the one-time symbolic differentiation and graph-construction overhead. This isolates the online cost of evaluating higher-order derivatives when their expressions are built offline. Compared with Fig.~\ref{fig:IDbenchmark} and Fig.~\ref{fig:FDbenchmark}, {{the results suggest that the proposed algorithms can provide computational advantages over automatic differentiation in the considered benchmarks when evaluating geometrically exact higher-order time derivatives of large floating-base tree dynamics.}}}

\subsection{Initialization}
The robot has a base with a relative configuration $\mathbf{C}^0_{B_1} \in SE(3)$, denoted by $B_1$ and equipped with 6 rigidly attached and tilted propellers labeled $p_1$ to $p_6$~\cite{TilThex2015}. Two manipulator arms are attached to the base vertically at a distance $b$ below its CoM location. One arm has 3 revolute joints labeled $q_2,\,q_3,$ and $q_4$ $\in \mathbb{S}^1$ connecting $B_1$ to $B_2$, $B_2$ to $B_3$, and $B_3$ to $B_4$, respectively. The other arm is similarly attached, as depicted in Fig.~\ref{main_vehicle}, forming an aerial manipulator with 2 parallel arms. Frame $\mathcal{F}_{B_i}:=\{x_{B_i},y_{B_i},z_{B_i}\}$ is placed at $B_i$'s CoM, while $\mathcal{F}_{w}$ is the inertial frame. The configuration of the system is thus described by $\mathcal{Q}=(\mathbf{C}^0_{B_1},\boldsymbol{q})\in SE(3)\times \mathbb{T}^6$, while the configuration velocity is $\mathcal{V}=(\boldsymbol{V}^0_{B_1},\dot{\boldsymbol{q}})\in se(3)\times T_q\mathbb{T}^6\cong se(3)\times\mathbb{R}^6$, where $\boldsymbol{V}^0_{B_1}$ is the spatial (right-invariant) base twist.

We start by laying down the description of the system needed for the initialization of the algorithms. The graph has 2 branches, each having 3 nodes, amounting to 7 bodies ($N=7$). Let $P(1,j)$, $E_{P(1,j)}$, $c(j)$, and $p(j)$ denote the ordered set of nodes (bodies) from $B_1$ to $B_j$ in Fig.~\ref{main_vehicle}(b), the associated directed path (joints) set, the child set of body $B_j$, and the parent set of body $B_j$, respectively. For each $j \in\{1,\cdots,7\}$, they are given by  
\begin{equation}
\resizebox{1\columnwidth}{!}{$
\begin{aligned}
c(1)&=\{2,5\},\, c(2)=\{3\},\,c(3)=\{4\},\,c(5)=\{6\},\,c(6)=\{7\},\, c(7)=\{\phi\},\\
p(1)&=\{\phi\},\, p(2)=\{1\},\,p(3)=\{2\},\,p(4)=\{3\},\,p(5)=\{1\},\,p(6)=\{5\},\,p(7)=\{6\},\\
P(1,1)&=\{B_1\},\,P(1,2)=\{B_1,B_2\},\,P(1,3)=\{B_1,B_2,B_3\},\\
P(1,4)&=\{B_1,B_2,B_3,B_4\},\,P(1,5)=\{B_1,B_5\},\,P(1,6)=\{B_1,B_5,B_6\},\\
P(1,7)&=\{B_1,B_5,B_6,B_7\},\,E_{P(1,2)}=\{q_2\},\,E_{P(1,3)}=\{q_2,q_3\},\\
E_{P(1,4)}&=\{q_2,q_3,q_4\},\,E_{P(1,5)}=\{q_5\},\,E_{P(1,6)}=\{q_5,q_6\},\,E_{P(1,7)}=\{q_5,q_6,q_7\}.
\end{aligned}
$}
\end{equation}

The wrench of propeller $p_i$ applied at and expressed in the base frame $\mathcal{F}_{B_1}$ is given by $\boldsymbol{W}^{B_1}_{1,p_i}\in se(3)^{*}$:
\begin{equation}
\resizebox{1\columnwidth}{!}{$
\begin{aligned}
\boldsymbol{W}^{B_1}_{1,p_i}&=\boldsymbol{\mathrm{Ad}}^{-T}_{\mathbf{C}^{B_1}_{p_i}} \boldsymbol{W}^{p_i}_{1,p_i},\qquad 
\mathbf{C}^{B_1}_{p_i}=\mathbf{C}_{R_z}((2-i)\pi/3)\mathbf{C}_{T_x}(a)\mathbf{C}_{R_x}(\alpha)\mathbf{C}_{R_y}(\beta),\\
&i=\{1,\cdots,6\},\qquad
\boldsymbol{W}^{B_1}_{B_1,\mathrm{prop}}= \sum_{i=1}^6\boldsymbol{W}^{B_1}_{1,p_i}.
\end{aligned}
$}
\end{equation}
Here, $\alpha,\beta \in \mathbb{S}^1$ are fixed angles, $\mathbf{C}_{R_e}(w)$ denotes a pure rotation transformation around axis $e$ with angle $w$, and $\mathbf{C}_{T_e}(w)$ denotes a pure translation transformation along axis $e$ with displacement $w$. The local propeller wrench is $\boldsymbol{W}^{p_i}_{1,p_i}=(\,0\;\;0\;\;c_{d_i}\;\;0\;\;0\;\;c_{f_i}\,)^T \omega_{p_i}^2$, where $c_f$ and $c_d$ are the thrust and drag coefficients, respectively, and $\omega_{p_i}$ is the propeller speed (control input). The total propeller wrench in the base frame is denoted $\boldsymbol{W}^{B_1}_{B_1,\mathrm{prop}}$. Moreover, joint $q_j$ can have several sources of external torques applied at it, the summation of which is denoted by $\Bar{\boldsymbol{\tau}}_{j}$.

Initially, at home configuration $\mathcal{H}_{0}$ shown in Fig.~\ref{main_vehicle}(a), all frames $\mathcal{F}_{B_j}$, $j=\{1,..,7\}$, are oriented identically with the world frame $\mathcal{F}_{w}$ for ease of computation. Thus, the constant home configurations of $\mathbf{A}^{p(B_j)}_{B_j}$ are, with $i=\{2,5\}$ and $k=\{3,4,6,7\}$:
\begin{equation}
\mathbf{A}^{B_{p(i)}}_{B_i}= \mathbf{C}_{T_z}(-b-d)\mathbf{C}_{T_x}\big((-1)^{i}\tfrac{c}{2}\big),\qquad
\mathbf{A}^{B_{p(k)}}_{B_k}= \mathbf{C}_{T_z}(-2d).
\end{equation}

Finally, let $\Bar{\boldsymbol{p}}_j=(0,\,0,\,d)^T$ be the position vector in frame $\mathcal{F}_{B_j}$, $j=\{2,..,7\}$, of a point on the respective joint axis, expressed in the inertial frame $\mathcal{F}_{w}$ at the home configuration $\mathcal{H}_{0}$:
\begin{equation}
\begin{split}
\boldsymbol{e}_i=(0,\,1,\,0)^T,\qquad 
\boldsymbol{e}_k=(1,\,0,\,0)^T,\qquad 
\boldsymbol{e}_l=(0,\,0,\,1)^T,
\end{split}
\end{equation}
where $i=\{2,5\}$, $k=\{3,6\}$, and $l=\{4,7\}$.

\begin{table}
\centering
\caption{{{Simulation Parameters.}}}
\label{tab:tilthex_sim_params}
\scriptsize

\setlength{\tabcolsep}{0.5pt}
\renewcommand{\arraystretch}{1}
\begin{tabularx}{\columnwidth}{@{}l l X@{}}
\toprule
\textbf{Category} & \textbf{Symbol} & \textbf{Value / definition} \\
\midrule

\multicolumn{3}{@{}l}{\textit{Simulation horizon and dynamics orders}}\\
Simulation horizon & $T$ & $30~\mathrm{s}$ \\
Time step & $\Delta t$ & $0.01~\mathrm{s}$ \\
Number of algorithm calls & $N_t$ & $3000$ \\
Orders & $r$ & $0,1,2,3,4,5$ \\
\midrule

\multicolumn{3}{@{}l}{\textit{Base reference pose (used to generate $\mathbf{C}^0_{B_1}(t)$)}}\\
Position: Planar circle at $z=0$ with radius & $-$ & $1.0~\mathrm{m}$ \\
Position: Trajectory phase & $-$ & $2\pi/20~\mathrm{rad/s}$ \\
Attitude: Yaw angular frequency & $-$ & $2\pi/30~\mathrm{rad/s}$ \\
Attitude: Yaw amplitude & $-$ & $25^\circ$ \\
\midrule

\multicolumn{3}{@{}l}{\textit{Joint reference trajectories (half-cosine), $j\in\{2,\dots,7\}$}}\\
Trajectory form & $q_j(t)$ &
$q_{0,j}+\dfrac{\eta_j}{2}\bigl(1-\cos(\pi t/T)\bigr)$ \\
Initial joint angles & $\boldsymbol{q}_0$ &
$\mathrm{deg2rad}\!\left([\,0,\ 10,\ -5,\ 15,\ -10,\ 7,\ -12\,]\right)$ \\
Half-cosine amplitudes & $\boldsymbol{\eta}$ &
$\mathrm{deg2rad}\!\left([\,0,\ 60,\ 50,\ 40,\ 55,\ 45,\ 35\,]\right)$ \\
\midrule

\multicolumn{3}{@{}l}{\textit{Geometry at home configuration (used in $\mathbf{A}^{B_{p(j)}}_{B_j}$ and $\mathbf{Y}_j$)}}\\
Vertical offsets & $b,\ d$ & $b=0.10~\mathrm{m},\ d=0.06~\mathrm{m}$ \\
Arm lateral spacing & $c$ & $c=0.18~\mathrm{m}$ \\
Base home constant transform & $\mathbf{A}^0_{B_1}$ & $\mathbf{A}^0_{B_1}=\mathbf{I}_4$ \\
\midrule

\multicolumn{3}{@{}l}{\textit{Inertial parameters (used in $\,\mathbf{M}^b_{B_j}$), with $j\in\{2,\dots,7\}$}}\\
Base mass & $m_{B_1}$ & $2.5~\mathrm{kg}$ \\
Base angular moment of inertia & $\mathbf{J}_{B_1}$ & $\mathrm{diag}(0.03,\,0.03,\,0.05)\ \mathrm{kg\,m^2}$ \\
Link mass & $m_{B_j}$ & $0.25~\mathrm{kg}$ \\
Link angular moment of inertia & $\mathbf{J}_{B_j}$ & $\mathrm{diag}(0.002,\,0.002,\,0.001)\ \mathrm{kg\,m^2}$ \\
Body-fixed inertia tensor at CoM & $\mathbf{M}^b_{B_j}$ &
$\mathbf{M}^b_{B_j}=
\begin{bmatrix}
\mathbf{J}_{B_j} & \mathbf{0}\\
\mathbf{0} & m_{B_j}\mathbf{I}_3
\end{bmatrix}$ \\
\midrule

\multicolumn{3}{@{}l}{\textit{External wrenches/torques $\forall j$}}\\
External body wrenches & $\boldsymbol{W}^{0\,(k)}_{B_j,\mathrm{app}}$ &
$\boldsymbol{W}^{0\,(k)}_{B_j,\mathrm{app}}=\mathbf{0},\quad k=0,\dots,r$  \\
\bottomrule
\end{tabularx}
\end{table}

\begin{table*}[h]
\centering
\caption{{{Benchmark setup.} }
{{Model: aerial manipulator in Fig. \ref{main_vehicle} with $N_B{=}5$ branches; kinematic and dynamics parameters are in Table \ref{tab:tilthex_sim_params}.
CasADi benchmark measures \emph{evaluation time} of the r-differentiated dynamics; Automatic differentiation occurs once per $r>0$ via built-in \texttt{jtimes(dy,t,1)} outside the timed loop.}}}
\label{tab:bench_params_short}
\scriptsize

\setlength{\tabcolsep}{4pt}
\renewcommand{\arraystretch}{1.15}
\begin{tabularx}{\textwidth}{@{}l l X@{}}
\toprule
\textbf{Benchmark} & \textbf{Setting} & \textbf{Value / definition} \\
\midrule

H-GRNE-FBS & Sweeps &
vs $N$: $N=1+5N_{\mathrm{BOD}}$, $N_{\mathrm{BOD}}\in[1,199] \text{ per branch}$ ($N\le 996$), total $100$ sweeps per $r$, $r\in\{0,\dots,5\}$; 
vs $r$: fixed $N=101$, $r\in\{0,\dots,10\}$. \\
 H-GRNE-FBS & Inputs &
\texttt{(C0\_1,V0\_1\_derivs,q\_derivs,W\_app,r)} with $C0\_1=\mathbf{I}_4$, $W_{\mathrm{app}}=\mathbf{0}$, 
$V0\_1\_derivs\in\mathbb{R}^{6\times12}$ and $q\_derivs\in\mathbb{R}^{N\times13}$ chosen randomly per sweep. \\

H-GABI-FBS & Sweeps &
vs $N$: same $N$ range, $r\in\{0,\dots,5\}$; vs $r$: fixed $N=101$, $r\in\{0,\dots,10\}$. \\
H-GABI-FBS & Inputs &
\texttt{(C0\_1,V0\_1,q0,qdot,Wprop,tau\_bar,W\_app,r)} with $C0\_1=\mathbf{I}_4$, $V0\_1=\mathbf{0}$,
$q,\dot q$ randomly chosen per sweep, and $\tau=\mathbf{0}$, $W_{\mathrm{prop}}=\mathbf{0}$, $W_{\mathrm{app}}=\mathbf{0}$. \\

CasADi ID & Problem size / orders &
Fixed $N=101$, $r\in\{0,\dots,8\}$, $N_{\mathrm{rep}}=1000$, warmup $=5$, JIT disabled. \\
CasADi ID & Timing &
For each $r$: build \texttt{Function f(t,p)\{dy\}} once and time repeated evaluation \texttt{f(\{t0,Pdm[i]\})}.
Here $dy=\frac{d^r}{dt^r}y(t)$ obtained by repeated built-in \texttt{jtimes(dy,t,1)} (performed \emph{outside} the timing loop). \\

CasADi FD & Problem size / orders &
Fixed $N=101$, $r\in\{0,\dots,8\}$, $N_{\mathrm{rep}}=1000$, warmup $=5$, JIT disabled. \\
CasADi FD & Timing &
For each $r$: time evaluation of \texttt{Function f(t,p)\{dy\}} at $t_0=0$ with random $\mathbf{p}$; 
$dy=\frac{d^r}{dt^r}y(t)$ generated once per $r$ via \texttt{jtimes}. The FD graph includes CasADi built-ins
\texttt{mtimes} and \texttt{solve} (for base twist). \\

All & Timing metric &
Mean seconds/call: \texttt{elapsed\_time/1000}. \\
\bottomrule
\end{tabularx}
\end{table*}

\subsection{H-GRNE-FBS Algorithm $(r=\{0,1\})$: Forward Recursion}
Algorithm~\ref{alg:kine} yields $\mathbf{C}^0_{B_j}$ and the $(r+2)$-th order kinematics $(\boldsymbol{V}^0_{B_j})^{(r+1)}$ and instantaneous screws $\boldsymbol{S}_{q_j}^{(r+1)}$ for each $j=\{2,\cdots,7\}$. The forward kinematics of each body up to the 3rd order are computed by evaluating the expressions, for each $j\in P(1,4)$ and $j\in P(1,7)$ in parallel starting from 1 until reaching 4 and 7, respectively, shown in Table~\ref{tab:kinematics_orders}.
The constant spatial (right-invariant) screws $\mathbf{Y}_j\in se(3)$, measured in the home configuration $\mathcal{H}_{0}$ and represented in the inertial frame $\mathcal{F}_{w}$, are computed using~\eqref{constantScrew} from quantities in the 0-th order kinematics. Since they are constant, they only need to be computed once. The time derivatives of $\mathbf{M}^0_{B_j}$ are obtained from~\eqref{Mass}. At the end of the backward recursion, the inverse dynamics and their first derivatives for each body are obtained.
\subsection{H-GRNE-FBS Algorithm $(r=\{0,1\})$: Backward Recursion}
Running Algorithm~\ref{alg:Dyn} up to the 1st order yields the equations of motion and their first time derivatives by executing these computations, displayed in Table~\ref{tab:inverse_dynamics}, in parallel for each $j\in P(4,1)$ and $j\in P(7,1)$, starting from 4 and 7, respectively, and ending at 1.

The time derivatives of the spatial momentum $\boldsymbol{\Pi}^0_{B_{j}}$, the gravitational wrench $\boldsymbol{W}_{\mathrm{grav},B_{j}}^0$, and the external wrench $\boldsymbol{W}_{\mathrm{ext},B_{j}}^0$ applied at body $B_{j}$ are given by~\eqref{Momu},~\eqref{gravweriv}, and~\eqref{extweriv}, respectively. Observe that the 1-st order inverse dynamics block requires computing the 0-th order recursion first, as $\boldsymbol{Q}_{B_j}$ is passed from $r=0$ to $r=1$.  
As mentioned in Sec.~\ref{relation}, the equations of motion can be written in the conventional EoM form using the child and parent Boolean matrices $\mathbf{G}_c$ and $\mathbf{G}_p$:
\begin{equation}
\mathbf{G}_p=
\begin{pmatrix}
\mathbf{I}_{6} & \mathbf{0}_{6} & \mathbf{0}_{6} & \mathbf{0}_{6} & \mathbf{0}_{6} & \mathbf{0}_{6} & \mathbf{0}_{6} \\
\mathbf{I}_{6} & \mathbf{I}_{6} & \mathbf{0}_{6} & \mathbf{0}_{6} & \mathbf{0}_{6} & \mathbf{0}_{6} & \mathbf{0}_{6} \\
\mathbf{I}_{6} & \mathbf{I}_{6} & \mathbf{I}_{6} & \mathbf{0}_{6} & \mathbf{0}_{6} & \mathbf{0}_{6} & \mathbf{0}_{6} \\
\mathbf{I}_{6} & \mathbf{I}_{6} & \mathbf{I}_{6} & \mathbf{I}_{6} & \mathbf{0}_{6} & \mathbf{0}_{6} & \mathbf{0}_{6} \\
\mathbf{I}_{6} & \mathbf{0}_{6} & \mathbf{0}_{6} & \mathbf{0}_{6} & \mathbf{I}_{6} & \mathbf{0}_{6} & \mathbf{0}_{6} \\
\mathbf{I}_{6} & \mathbf{0}_{6} & \mathbf{0}_{6} & \mathbf{0}_{6} & \mathbf{I}_{6} & \mathbf{I}_{6} & \mathbf{0}_{6} \\
\mathbf{I}_{6} & \mathbf{0}_{6} & \mathbf{0}_{6} & \mathbf{0}_{6} & \mathbf{I}_{6} & \mathbf{I}_{6} & \mathbf{I}_{6}
\end{pmatrix},
\qquad
\mathbf{G}_c = \mathbf{G}_p^T.
\end{equation}
\subsection{H-GABI-FBS Algorithm $(r=\{0,1\})$: Backward Recursion}
Running Algorithms~\ref{alg:ForwardDyn} yields the 0-th order forward dynamics $(\dot{\boldsymbol{V}}_{B_{1}}^0,\ddot{\boldsymbol{q}}_{j})$ and their time derivatives $(\ddot{\boldsymbol{V}}_{B_{1}}^0,\dddot{\boldsymbol{q}}_{j})$ for each body in the graph of Fig.~\ref{main_vehicle}(b). Table~\ref{tab:fd_backward_recursion} illustrates the backward recursion, executed in parallel for each $j\in P(4,1)$ and $j\in P(7,1)$, starting from 4 and 7, respectively, and ending at 1.

Observe that the 1-st order forward dynamics block requires calculating both the backward and forward recursions of the 0-th order forward dynamics first, as $(\dot{\boldsymbol{V}}_{B_{1}}^0,\ddot{\boldsymbol{q}}_{j})$ and $(\mathbf{M}_{B_{j}}^A,\boldsymbol{W}_{B_{j}}^A)$ are passed forward from $r=0$ to $r=1$.

\subsection{H-GABI-FBS Algorithm $(r=\{0,1\})$: Forward Recursion}
The algorithm concludes by executing the expressions in Table~\ref{tab:fd_forward_recursion} for each $j\in P(1,4)$ and $j\in P(1,7)$ in parallel, starting from 1 until reaching 4 and 7, respectively.

\begin{figure*}
    \centering
    \begin{subfigure}{1.01\linewidth}
        \centering
        \includegraphics[width=1.015\linewidth]{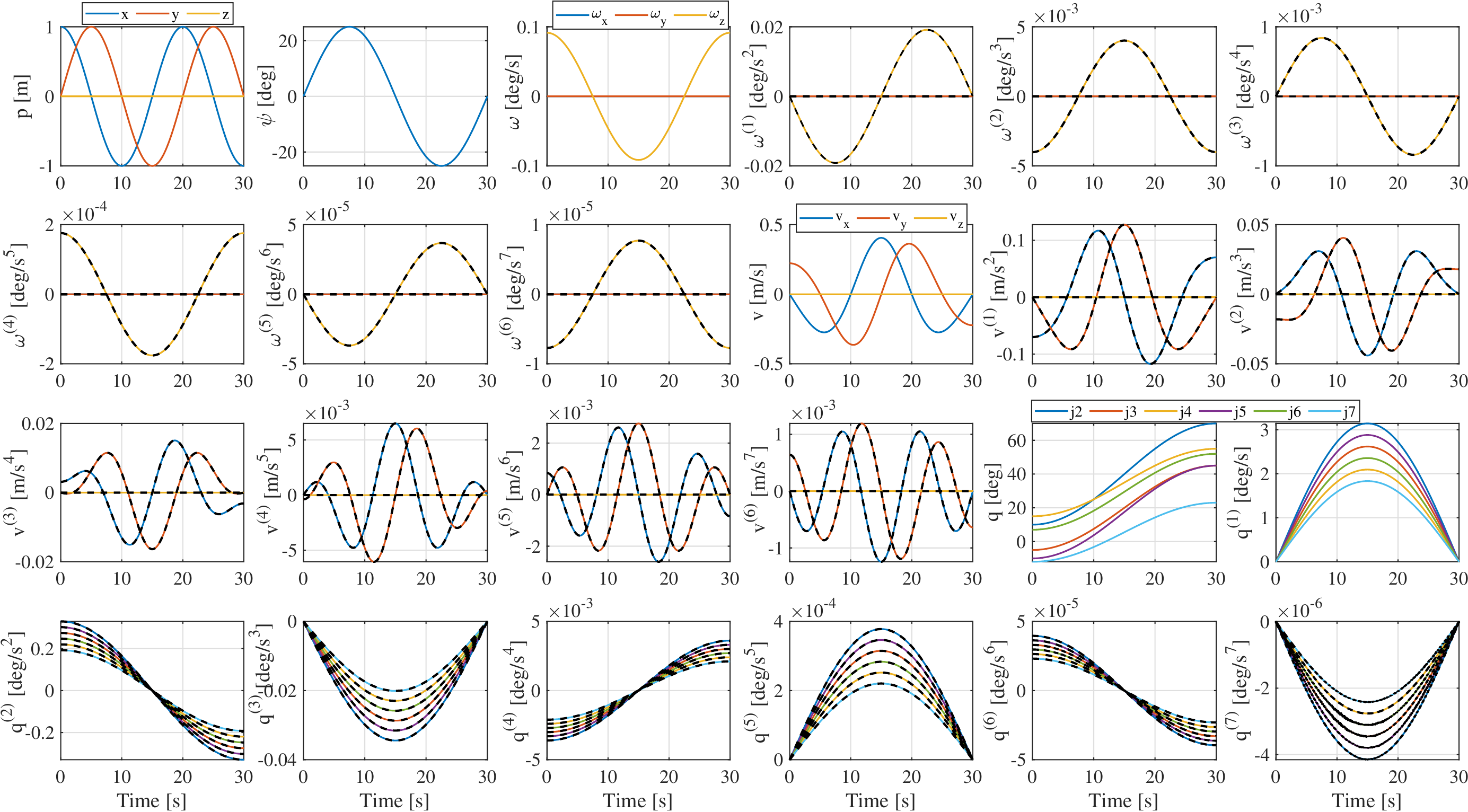}
        \caption{Input to H-GRNE-FBS (solid) and output of H-GABI-FBS (dashed) for $r=0$ to $r=5$}
        \label{fig:IDinput}
    \end{subfigure}
    \vspace{0.5em}
    \begin{subfigure}{1.01\linewidth}
        \centering
        \includegraphics[width=1.015\linewidth]{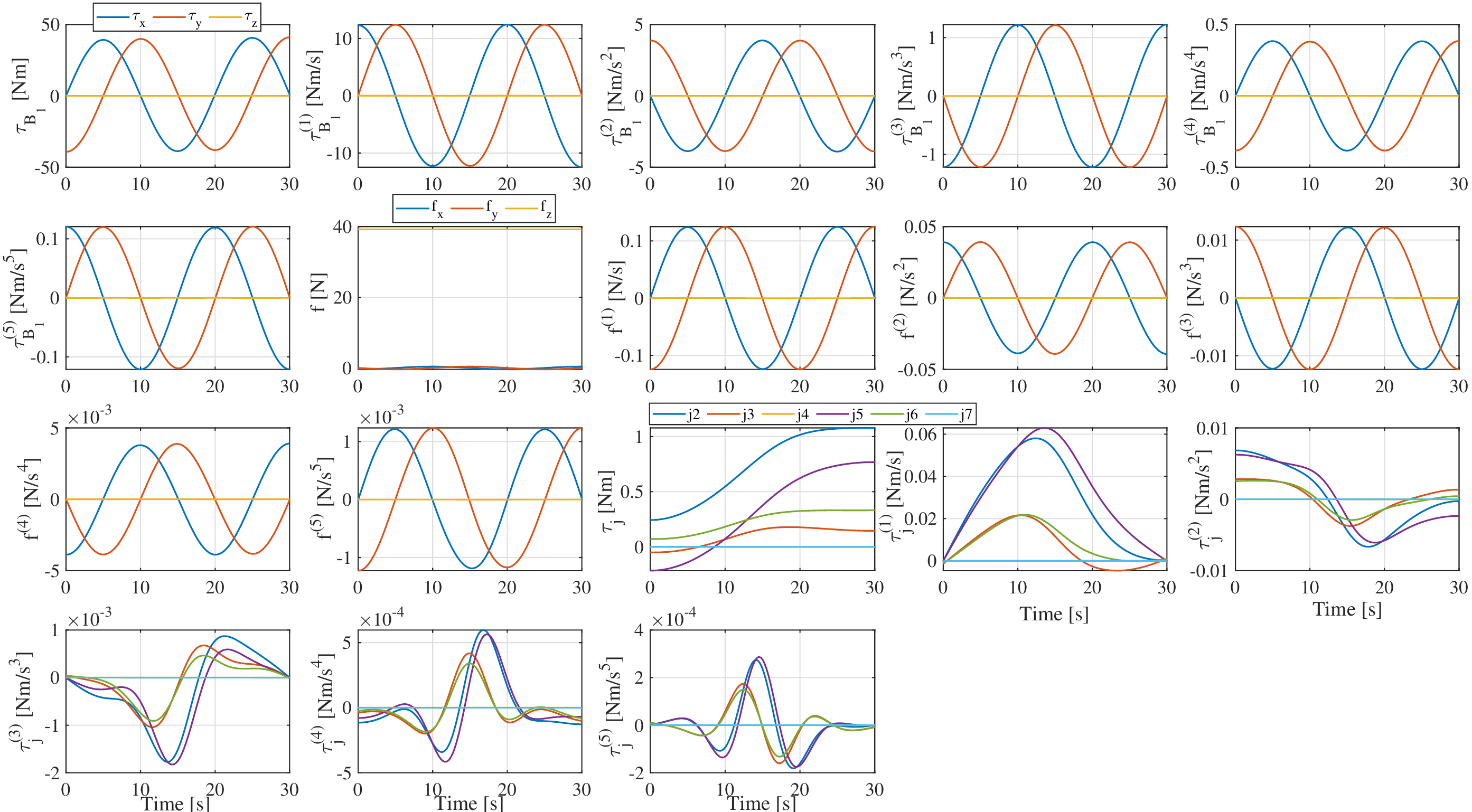}
        \caption{Output of H-GRNE-FBS (which is given as input to H-GABI-FBS as well) for $r=0$ to $r=5$}
        \label{fig:IDoutput}
    \end{subfigure}
    
    \caption{{{Inputs and outputs of H-GRNE-FBS and H-GABI-FBS up to order $r=5$. $\boldsymbol{W}^{0}_{B_1,\mathrm{prop}}=(\tau_{B_1},f)$ is obtained via $Q_1$ while $\tau_j$ is joint $j$ torque given by $Q_j$ from Alog. \ref{alg:Dyn}. Output of H-GABI-FBS up to r=5 is shown in dashed lines in (a). Observe that it retrieved back the input trajectories, as demonstrated by the dashed trajectories coinciding with their solid counterparts in (a), when given the state $(C^0_{B_1}(t),V^0_{B_1}(t),q(t))$, base wrench and joint torques in (b), verifying successful operation of our higher-order algorithms. For $r=0,\ldots,5$, the average runtime per call for H-GRNE-FBS is $(2.326,3.078,2.867,4.355,5.960,8.171)[\mu\text{s}]$, respectively, while for H-GABI-FBS it is $(4.106,4.802,6.486,6.461,8.589,11.085)[\mu\text{s}]$, respectively.}}}
    \label{fig:IDinout}
\end{figure*}
\begin{figure}[!t]
    \centering
    \begin{subfigure}{1\columnwidth}
        \centering
        \includegraphics[width=1\columnwidth]{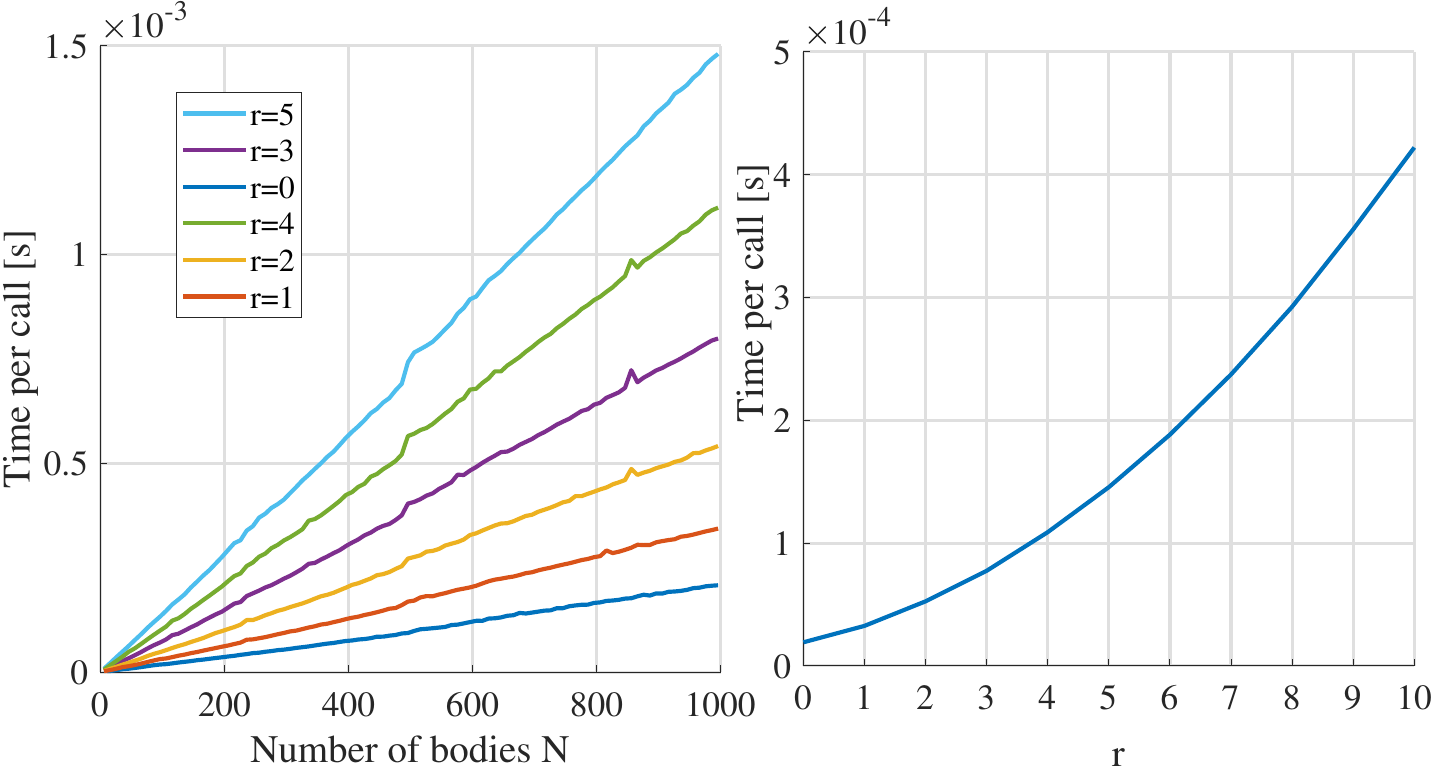}
        \caption{H-GRNE-FBS}
        \label{fig:IDbenchmark}
    \end{subfigure}
    \vspace{0.5em}
    \begin{subfigure}{1\columnwidth}
        \centering
        \includegraphics[width=1\columnwidth]{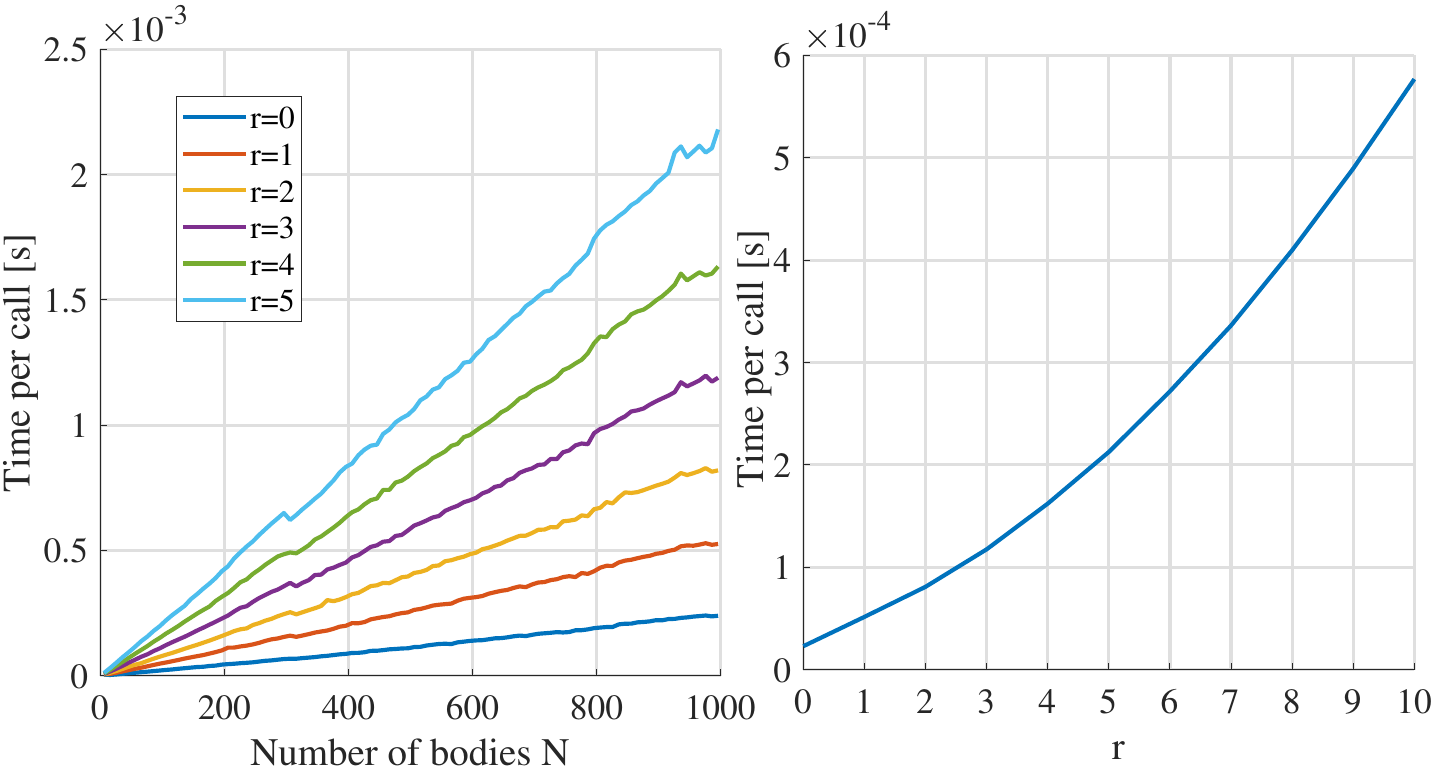}
        \caption{H-GABI-FBS}
        \label{fig:FDbenchmark}
    \end{subfigure}
\vspace{0.5em}
\begin{subfigure}{1\columnwidth}
        \centering
    \includegraphics[width=0.6\columnwidth]{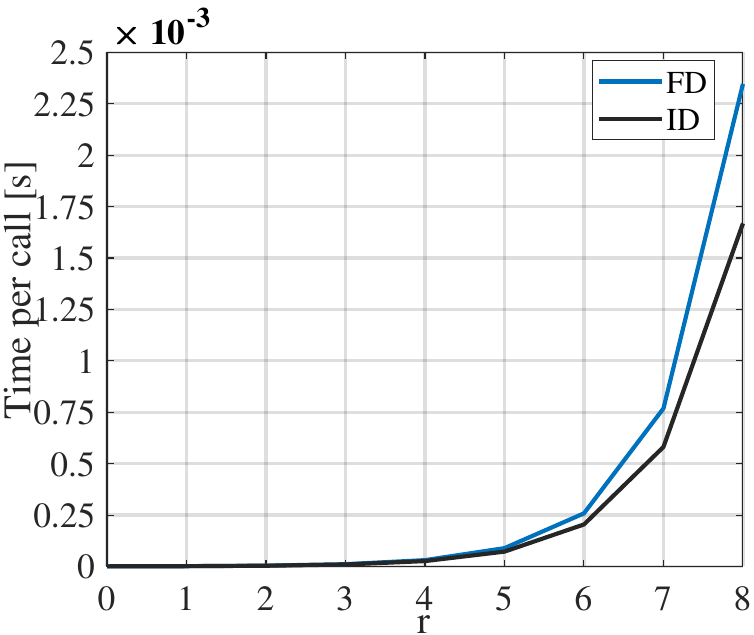}
    \caption{CasADi}
    \label{fig:CasadAD}
    \end{subfigure}
    \caption{{{Benchmark results demonstrate that time differentiating geometric dynamics of floating-base robots through our proposed methods becomes significantly faster than automatic differentiation when the order increases for the same tree size.}}}
    \label{fig:benchmark}
\end{figure}

\begin{table*}
\centering
\caption{{{Application of H-GRNE-FBS and H-GABI-FBS to aerial manipulator in Fig. \ref{main_vehicle}}}}
\label{tab:recursive_all}

\begin{subtable}{\textwidth}
\centering
\caption{Recursive Kinematic Relations up to Third Order}
\label{tab:kinematics_orders}
\renewcommand{\arraystretch}{1.3}
\begin{tabular}{|p{0.47\textwidth}|p{0.47\textwidth}|}
\hline
\textbf{0-th Order (Position Kinematics)} & 
\textbf{1-st Order (Velocity Kinematics)} \\ 
\hline
$\mathbf{A}^0_{B_j} = \prod_{{B_{i}}\in P(1,j)} \mathbf{A}^{B_{p(i)}}_{B_{i}}, \quad 
\mathbf{F}_{B_{j}} = \exp(\mathbf{Y}_{j} q_{j}), \quad 
\mathbf{F}_{B_{1}} = \mathbf{I}_4$ &
$\boldsymbol{V}^{0}_{B_{j}} = \boldsymbol{V}^{0}_{B_{1}} + 
\sum_{{i}\in E_{P(1,j)}} \boldsymbol{S}_{q_i}\,\dot{q}_{i}$ \\[4pt]
$\mathbf{C}^0_{B_j} = \mathbf{C}^0_{B_1}
\!\left(\prod_{{B_{i}}\in P(1,j)} \mathbf{F}_{B_{i}}\right)
\!\mathbf{A}^0_{B_j}$ &
$\boldsymbol{\Pi}^{0}_{B_{j}} = \mathbf{M}^0_{B_j} \boldsymbol{V}^{0}_{B_j}$ \\[4pt]
$(p_j\,1)^T = \mathbf{A}^0_{B_j} (\Bar{\boldsymbol{p}}_j\,1)^T, \;
\mathbf{Y}_{j} = (\boldsymbol{e}_j, [p_j]\boldsymbol{e}_j), \;
\boldsymbol{S}_{q_j} = \boldsymbol{\mathrm{Ad}}_{\mathbf{C}^0_{B_j} (\mathbf{A}^0_{B_j})^{-1}}\, \mathbf{Y}_{j}$ &
$\dot{\boldsymbol{S}}_{q_j} = \boldsymbol{\mathrm{ad}}_{\boldsymbol{V}^{0}_{B_j}}\,\boldsymbol{S}_{q_j}$ \\
\hline
\textbf{2-nd Order (Acceleration Kinematics)} &
\textbf{3-rd Order (Jerk Kinematics)} \\ 
\hline
$\dot{\boldsymbol{V}}^{0}_{B_j} = \dot{\boldsymbol{V}}^{0}_{B_1} +  
\sum_{{i}\in E_{P(1,j)}} (\boldsymbol{S}_{q_i}\,\ddot{q}_{i} +  
\dot{\boldsymbol{S}}_{q_i}\,\dot{q}_{i})$ &
$\ddot{\boldsymbol{V}}^{0}_{B_j} = \ddot{\boldsymbol{V}}^{0}_{B_1} +
\sum_{{i}\in E_{P(1,j)}} (\boldsymbol{S}_{q_i}\,\dddot{q}_{i} +
2\,\dot{\boldsymbol{S}}_{q_i}\,\ddot{q}_{i} + \ddot{\boldsymbol{S}}_{q_i}\,\dot{q}_{i})$ \\[4pt]
$\dot{\boldsymbol{\Pi}}^{0}_{B_j} =
\mathbf{M}^0_{B_j} \dot{\boldsymbol{V}}^{0}_{B_j} +
\dot{\mathbf{M}}^0_{B_j} \boldsymbol{V}^{0}_{B_j}$ &
$\ddot{\boldsymbol{\Pi}}^{0}_{B_j} =
\mathbf{M}^0_{B_j} \ddot{\boldsymbol{V}}^{0}_{B_j} +
2 \dot{\mathbf{M}}^0_{B_j} \dot{\boldsymbol{V}}^{0}_{B_j} +
\ddot{\mathbf{M}}^0_{B_j} \boldsymbol{V}^{0}_{B_j}$ \\[4pt]
$\ddot{\boldsymbol{S}}_{q_j} = 
\boldsymbol{\mathrm{ad}}_{\dot{\boldsymbol{V}}^{0}_{B_j}} \boldsymbol{S}_{q_j} +
\boldsymbol{\mathrm{ad}}_{\boldsymbol{V}^{0}_{B_j}}\,\dot{\boldsymbol{S}}_{q_j}$ &
$\dddot{\boldsymbol{S}}_{q_j} =
\boldsymbol{\mathrm{ad}}_{\ddot{\boldsymbol{V}}^{0}_{B_j}} \boldsymbol{S}_{q_j} +
2\,\boldsymbol{\mathrm{ad}}_{\dot{\boldsymbol{V}}^{0}_{B_j}} \dot{\boldsymbol{S}}_{q_j} +
\boldsymbol{\mathrm{ad}}_{\boldsymbol{V}^{0}_{B_j}}\,\ddot{\boldsymbol{S}}_{q_j}$ \\
\hline
\end{tabular}
\end{subtable}

\vspace{0.8em}

\begin{subtable}{\textwidth}
\centering
\caption{Recursive Inverse Dynamics Relations up to First Order}
\label{tab:inverse_dynamics}
\renewcommand{\arraystretch}{1.3}
\begin{tabular}{|p{0.47\textwidth}|p{0.47\textwidth}|}
\hline
\textbf{0-th Order Inverse Dynamics} &
\textbf{1-st Order Inverse Dynamics} \\
\hline
$\Tilde{\boldsymbol{W}}_{j} =
\dot{\boldsymbol{\Pi}}^0_{B_{j}} -
\boldsymbol{W}_{\mathrm{grav},B_{j}}^0 -
\boldsymbol{W}_{\mathrm{ext},B_{j}}^0, \quad
\boldsymbol{Q}_{j} = \Tilde{\boldsymbol{W}}_{j} +
\sum_{{i}\in c({j})} \boldsymbol{Q}_{i}$ &
$\dot{\Tilde{\boldsymbol{W}}}_{j} =
\ddot{\boldsymbol{\Pi}}^0_{B_{j}} -
\dot{\boldsymbol{W}}_{\mathrm{grav},B_{j}}^0 -
\dot{\boldsymbol{W}}_{\mathrm{ext},B_{j}}^0, \quad
\dot{\boldsymbol{Q}}_{j} = \dot{\Tilde{\boldsymbol{W}}}_{j} +
\sum_{{i}\in c({j})} \dot{\boldsymbol{Q}}_{i}$ \\[6pt]
$\Bar{\boldsymbol{\tau}}_{j} =
\boldsymbol{S}_{q_{j}}^{T} \boldsymbol{Q}_{j}$ &
$\dot{\Bar{\boldsymbol{\tau}}}_{j} =
\dot{\boldsymbol{S}}_{q_{j}}^{T} \boldsymbol{Q}_{j} +
\boldsymbol{S}_{q_{j}}^{T} \dot{\boldsymbol{Q}}_{j}$ \\[6pt]
\text{When $j=1$:}\;
$\boldsymbol{W}^{B_1}_{B_1,\mathrm{prop}} =
\boldsymbol{\mathrm{Ad}}^{T}_{\mathbf{C}^{0}_{B_1}} \boldsymbol{Q}_{B_1}$ &
\text{When $j=1$:}\;
$\dot{\boldsymbol{W}}^{B_1}_{B_1,\mathrm{prop}} =
\boldsymbol{\mathrm{Ad}}^{T}_{\mathbf{C}^{0}_{B_1}}
\!\left( \boldsymbol{\mathrm{ad}}^{T}_{\boldsymbol{V}^{0}_{B_1}}
\boldsymbol{Q}_{B_1} + \dot{\boldsymbol{Q}}_{B_1} \right)$ \\[6pt]
\hline
\end{tabular}
\end{subtable}

\vspace{0.8em}

\begin{subtable}{\textwidth}
\centering
\caption{Recursive Forward Dynamics (Backward Recursion) up to First Order}
\label{tab:fd_backward_recursion}
\renewcommand{\arraystretch}{1.1}
\resizebox{\textwidth}{!}{%
\begin{tabular}{|p{0.47\textwidth}|p{0.47\textwidth}|}
\hline
\textbf{0-th Order Forward Dynamics (Backward Recursion)} &
\textbf{1-st Order Forward Dynamics (Backward Recursion)} \\
\hline
$\begin{aligned}
\Tilde{\boldsymbol{W}}_{B_j} &= 
-\,\boldsymbol{\mathrm{ad}}^{T}_{\boldsymbol{V}^{0}_{B_j}}\,\mathbf{M}^0_{B_j}\boldsymbol{V}^{0}_{B_j}
 - \boldsymbol{W}_{\mathrm{grav},B_{j}}^0 - \boldsymbol{W}_{\mathrm{ext},B_{j}}^0,\\
\Tilde{\mathbf{M}}_{B_j} &= \mathbf{M}^0_{B_j},\\
\mathbf{M}_{B_j}^A &= \Tilde{\mathbf{M}}_{B_j} 
+ \sum_{{i}\in c({j})}\!\left(
\mathbf{M}_{B_i}^A - 
\mathbf{M}_{B_i}^A \boldsymbol{S}_{q_i}
\big(\boldsymbol{S}_{q_i}^T \mathbf{M}_{B_i}^A \boldsymbol{S}_{q_i}\big)^{-1}
\boldsymbol{S}_{q_i}^T \mathbf{M}_{B_i}^A
\right),\\
\boldsymbol{W}_{B_j}^A &= 
\Tilde{\boldsymbol{W}}_{B_j} + 
\sum_{{i}\in c({j})} \boldsymbol{W}_{B_i}^A
+ \mathbf{M}_{B_i}^A \!\left( 
\boldsymbol{S}_{q_i}\,\ddot{\Tilde{q}}_{i} + 
\dot{\boldsymbol{S}}_{q_i}\,\dot{q}_{i}\right),\\
\ddot{\Tilde{q}}_{i} &= 
\big(\boldsymbol{S}_{q_i}^T \mathbf{M}^A_{B_i} \boldsymbol{S}_{q_i}\big)^{-1}
\!\left(\Bar{\boldsymbol{\tau}}_{i} - 
\boldsymbol{S}_{q_i}^T \big(\mathbf{M}^A_{B_i}\dot{\boldsymbol{S}}_{q_i}\dot{q}_{i}+ 
\boldsymbol{W}_{B_i}^A\big)\right)
\end{aligned}$ &
$\begin{aligned}
\dot{\Tilde{\boldsymbol{W}}}_{B_j} &= 
2\,\dot{\mathbf{M}}^0_{B_j}\,\dot{\boldsymbol{V}}_{B_j}^0
 + \ddot{\mathbf{M}}^0_{B_j}\,\boldsymbol{V}_{B_j}^0
 - \dot{\boldsymbol{W}}_{\mathrm{grav},B_{j}}^0
 - \dot{\boldsymbol{W}}_{\mathrm{ext},B_{j}}^0,\\
\dot{\boldsymbol{W}}_{B_j}^A &= 
\dot{\Tilde{\boldsymbol{W}}}_{B_j}
+ \sum_{{i}\in c({j})}\dot{\boldsymbol{W}}_{B_i}^A
+ \mathbf{M}_{B_i}^A \!\left(
\boldsymbol{S}_{q_i}\,\dddot{\Tilde{q}}_{i}
 + 2\,\dot{\boldsymbol{S}}_{q_i}\,\ddot{q}_{i}
 + \ddot{\boldsymbol{S}}_{q_i}\,\dot{q}_{i}\right),\\
\dddot{\Tilde{q}}_{i} &= 
\big(\boldsymbol{S}_{q_i}^T \mathbf{M}^A_{B_i} \boldsymbol{S}_{q_i}\big)^{-1}\!\Big(
\dot{\Bar{\boldsymbol{\tau}}}_{i}-\dot{\Tilde{\boldsymbol{\tau}}}_{i}
 - \boldsymbol{S}_{q_i}^T\big(
 \mathbf{M}^A_{B_i}\,(2\,\dot{\boldsymbol{S}}_{q_i}\,\ddot{q}_{i}
 \\& \hspace{5cm}+ \ddot{\boldsymbol{S}}_{q_i}\,\dot{q}_{i})+
 \dot{\boldsymbol{W}}_{B_i}^A\big)
\Big),\\
\dot{\Tilde{\boldsymbol{\tau}}}_{i} &= 
\dot{\boldsymbol{S}}_{q_i}^T\big(
\mathbf{M}_{B_i}^A\,\dot{\boldsymbol{V}}_{B_i}^0 +
\boldsymbol{W}_{B_i}^A\big)
\end{aligned}$
\\
\hline
\end{tabular}%
}
\end{subtable}

\vspace{0.8em}

\begin{subtable}{\textwidth}
\centering
\caption{Recursive Forward Dynamics (Forward Recursion) up to First Order}
\label{tab:fd_forward_recursion}
\renewcommand{\arraystretch}{1.1}
\resizebox{\textwidth}{!}{%
\begin{tabular}{|p{0.47\textwidth}|p{0.47\textwidth}|}
\hline
\textbf{0-th Order Forward Dynamics (Forward Recursion)} &
\textbf{1-st Order Forward Dynamics (Forward Recursion)} \\ 
\hline
$\begin{aligned}
&\text{when } j=1:\;
\dot{\boldsymbol{V}}_{B_1}^0 = 
(\mathbf{M}_{B_1}^A)^{-1}
\Big( \boldsymbol{\mathrm{Ad}}^{-T}_{\mathbf{C}^{0}_{B_1}}
\boldsymbol{W}^{B_1}_{B_1,\mathrm{prop}}
- \boldsymbol{W}_{B_1}^A \Big),\\
&\ddot{\boldsymbol{q}}_{j} =
- \big( \boldsymbol{S}_{q_j}^T \mathbf{M}^A_{B_j} \boldsymbol{S}_{q_j} \big)^{-1}
\boldsymbol{S}_{q_j}^T \mathbf{M}^A_{B_j}
\dot{\boldsymbol{V}}^0_{B_{p(j)}} + \ddot{\Tilde{\boldsymbol{q}}}_{j},\\
&\dot{\boldsymbol{V}}_{B_j}^0 =
\dot{\boldsymbol{V}}^{0}_{B_1} +
\sum_{{i}\in E_{P(1,j)}}
\big( \boldsymbol{S}_{q_i}\,\ddot{q}_{i} +
\dot{\boldsymbol{S}}_{q_i}\,\dot{q}_{i} \big)
\end{aligned}$ &
$\begin{aligned}
&\text{when } j=1:\;
\ddot{\boldsymbol{V}}_{B_1}^0 =
(\mathbf{M}_{B_1}^A)^{-1} \Big(
\boldsymbol{\mathrm{Ad}}^{-T}_{\mathbf{C}^{0}_{B_1}}
\dot{\boldsymbol{W}}^{B_1}_{B_1,\mathrm{prop}}
- \dot{\boldsymbol{W}}_{B_1}^A
\\& \hspace{3.5cm}- \boldsymbol{\mathrm{ad}}^{T}_{\boldsymbol{V}^{0}_{B_1}}
(\mathbf{M}_{B_1}^A \dot{\boldsymbol{V}}_{B_1}^0 +
\boldsymbol{W}_{B_1}^A) \Big),\\
&\dddot{\boldsymbol{q}}_{j} =
- \big( \boldsymbol{S}_{q_j}^T \mathbf{M}^A_{B_j} \boldsymbol{S}_{q_j} \big)^{-1}
\boldsymbol{S}_{q_j}^T \mathbf{M}^A_{B_j}
\ddot{\boldsymbol{V}}^0_{B_{p(j)}} +
\dddot{\Tilde{\boldsymbol{q}}}_{j},\\
&\ddot{\boldsymbol{V}}_{B_j}^0 =
\ddot{\boldsymbol{V}}^{0}_{B_1} +
\sum_{{i}\in E_{P(1,j)}}
\big( \boldsymbol{S}_{q_i}\,\dddot{q}_{i} +
2\,\dot{\boldsymbol{S}}_{q_i}\,\ddot{q}_{i} +
\ddot{\boldsymbol{S}}_{q_i}\,\dot{q}_{i} \big)
\end{aligned}$
\\
\hline 
\end{tabular}%
}
\end{subtable}

\end{table*}
\section{Conclusion}
\label{conc}
We presented novel extensions of the three widely-used $O(N)$ algorithms in the literature on rigid-body dynamics—the recursive Newton-Euler algorithm for inverse dynamics, the Articulated Body Inertia method for forward dynamics and hybrid dynamics algorithm—that compute geometrically-exact higher-order forward and inverse dynamics for floating-base open kinematics systems, where the base configuration is in $SE(3)$, in a compact and computationally-efficient form. This is accomplished drawing upon tools from lie groups and screw theory extending the approaches in \cite{Coordinate-inva,featherstone2014rigid,spatialtwist} which are intended for fixed-base manipulators or a floating-base system with some parameterization of $SE(3)$. Closed-form EoM and their first time derivatives were then derived from quantities evaluated in the recursive algorithm. We also showed how an admissible Coriolis matrix satisfying the passivity property can conveniently be constructed from such quantities. {{An interesting insight reducing computational burden is that the articulated body inertia assumes the same value in any order. Moreover, the developed theory is applied to simulate the higher order dynamics of a 12-DoF floating base system, showing a successful comparably fast implementation of the algorithms. Benchmarking indicates our methods can be applied in practice and preferable over automatic differentiation since they converge faster and scale better}}. Future work may focus on expanding the class of floating-base systems considered herein to include the case of floating-base systems with closed kinematics chains similar to \cite{closedKin}, and where the attached mechanism having a different configuration than the manifold $\mathbb{T}^{n_1}\times \mathbb{R}^{n_2}$, adapting these algorithms to compute the partial derivatives of the higher-order forward dynamics, and {{carrying a comparison against body-frame or hybrid-frame higher-order formulations for the considered class of robots to highlight the trade-offs in compactness and computational complexities.}}  
\section*{Acknowledgement}
This work was partially funded by the Horizon Europe research agreement no. 101120732 (AUTOASSESS) and by the NWO OTP project AVIATOR.






\selectlanguage{english} 



\bibliographystyle{asmejour}   

\bibliography{Bibliography} 

\appendix

\section{{{Proof of Lemma 1}}}
\label{append:proof_Lemma1}
    The proof follows from direct computation of the expression $\tfrac{1}{2}\dot{\Bar{\mathbf{M}}}-\mathbf{C}$ with $\Bar{\mathbf{M}}$ defined in \eqref{eqsprop}.
    \begin{equation}
    \begin{split}
        \tfrac{1}{2}\dot{\Bar{\mathbf{M}}}-\mathbf{C}=&\tfrac{1}{2}(\dot{\mathbf{S}}^T \mathbf{G}_c \mathbf{M}  \mathbf{G}_p \mathbf{S}+ \mathbf{S}^T \mathbf{G}_c \dot{\mathbf{M}}  \mathbf{G}_p \mathbf{S}+\mathbf{S}^T \mathbf{G}_c \mathbf{M}  \mathbf{G}_p \dot{\mathbf{S}})\\
        &-\mathbf{S}^T \mathbf{G}_c \mathbf{M} \mathbf{G}_p \dot{\mathbf{S}} + \mathbf{S}^T \mathbf{G}_c \boldsymbol{\mathrm{ad}}^T_{\mathbf{G}_p \mathbf{S} {{\boldsymbol{\nu}}}} \mathbf{M} \mathbf{G}_p \mathbf{S}\\
        =&\tfrac{1}{2}\Big(\dot{\mathbf{S}}^T \mathbf{G}_c \mathbf{M}  \mathbf{G}_p \mathbf{S}+\mathbf{S}^T \mathbf{G}_c \mathbf{M}  \mathbf{G}_p \dot{\mathbf{S}}\\
        &- \mathbf{S}^T \mathbf{G}_c (\mathbf{M}\boldsymbol{\mathrm{ad}}_{\boldsymbol{V}}+\boldsymbol{\mathrm{ad}}^T_{\boldsymbol{V}}\mathbf{M})  \mathbf{G}_p \mathbf{S}\Big)\\
        &-\mathbf{S}^T \mathbf{G}_c \mathbf{M} \mathbf{G}_p \dot{\mathbf{S}} +\mathbf{S}^T \mathbf{G}_c \boldsymbol{\mathrm{ad}}^T_{\boldsymbol{V}} \mathbf{M} \mathbf{G}_p \mathbf{S}.
    \end{split}
    \end{equation}
    A straightforward calculations show that $\mathbf{A}_1$ and $\mathbf{A}_2$ are skew-symmetric matrices, provided that $\mathbf{G}_c=\mathbf{G}_p^T$ and $\mathbf{M}^T=\mathbf{M}$, which is indeed the case for the multibody system in question. By using the fact that the difference of skew-symmetric matrices yields a skew-symmetric matrix, hence $\mathbf{A}_1-\mathbf{A}_2$ is skew-symmetric. As a result, $\tfrac{1}{2}\dot{\Bar{\mathbf{M}}}-\mathbf{C}$ is skew-symmetric.
\section{{{Proof of Proposition 1}}} 
\label{prop1_proof}
Observe that the 1st and 2nd order kinematics can be expressed as
\begin{align}
        \boldsymbol{V}&= \mathbf{G}_p\, \mathbf{S}\, {{\boldsymbol{\nu}}},\\
        \dot{\boldsymbol{V}}&= \mathbf{G}_p\, \mathbf{S}\, {{\dot{\boldsymbol{\nu}}}} + \mathbf{G}_p\, \boldsymbol{\mathrm{ad}}_{\mathbf{G}_p \mathbf{S} {{\boldsymbol{\nu}}}} \, \mathbf{S}\, {{\boldsymbol{\nu}}},\\
        \dot{\boldsymbol{\Pi}}&= \mathbf{G}_c \, \dot{\boldsymbol{\Pi}}_{I} 
        = \mathbf{G}_c \mathbf{M} \dot{\boldsymbol{V}} - \mathbf{G}_c \boldsymbol{\mathrm{ad}}^T_{\mathbf{G}_p \mathbf{S} {{\boldsymbol{\nu}}}} \mathbf{M} \boldsymbol{V} \nonumber\\
        &= \mathbf{G}_c \mathbf{M}  \mathbf{G}_p\, \mathbf{S}\, {{\dot{\boldsymbol{\nu}}}} + \mathbf{G}_c \mathbf{M} \mathbf{G}_p \boldsymbol{\mathrm{ad}}_{\mathbf{G}_p \mathbf{S} {{\boldsymbol{\nu}}}} \, \mathbf{S}\, {{\boldsymbol{\nu}}} - \mathbf{G}_c \boldsymbol{\mathrm{ad}}^T_{\mathbf{G}_p \mathbf{S} {{\boldsymbol{\nu}}}} \mathbf{M} \mathbf{G}_p\, \mathbf{S}\, {{\boldsymbol{\nu}}}.
\end{align}
For the 0-th order gravitational and external wrenches we have 
\begin{align}
        \boldsymbol{F}_{\mathrm{app}}&= - \mathbf{G}_c \,\boldsymbol{W}_{I,\mathrm{app}}^0,\,\,
        \boldsymbol{F}_{\mathrm{grav}}= - \mathbf{G}_c \,\boldsymbol{W}_{I,\mathrm{grav}}^0.
\end{align}
Thus, 0-th order inverse dynamics admits the closed form
\begin{equation}
\begin{aligned}
        \Bar{\boldsymbol{\tau}} &= \mathbf{S}^T(\dot{\boldsymbol{\Pi}} + \boldsymbol{F}_{\mathrm{app}} + \boldsymbol{F}_{\mathrm{grav}})\\
        &= \underbrace{\mathbf{S}^T \mathbf{G}_c \mathbf{M}  \mathbf{G}_p\, \mathbf{S}}_{\Bar{\mathbf{M}}({{\boldsymbol{\mathcal{Q}}}})}\, {{\dot{\boldsymbol{\nu}}}} 
        + \underbrace{
        \begin{aligned}[t]
        &\mathbf{S}^T \mathbf{G}_c \mathbf{M} \mathbf{G}_p \boldsymbol{\mathrm{ad}}_{\mathbf{G}_p \mathbf{S} {{\boldsymbol{\nu}}}} \mathbf{S}\, {{\boldsymbol{\nu}}} \\
        &- \mathbf{S}^T \mathbf{G}_c \boldsymbol{\mathrm{ad}}^T_{\mathbf{G}_p \mathbf{S} {{\boldsymbol{\nu}}}} \mathbf{M} \mathbf{G}_p\, \mathbf{S}\, {{\boldsymbol{\nu}}}
        \end{aligned}
        }_{\boldsymbol{h}({{\boldsymbol{\mathcal{Q}}}},{{\boldsymbol{\nu}}})} \\
        &- \underbrace{\mathbf{S}^T \mathbf{G}_c \,\boldsymbol{W}_{I,\mathrm{grav}}^0}_{\boldsymbol{g}({{\boldsymbol{\mathcal{Q}}}})}  
         - \underbrace{\mathbf{S}^T \mathbf{G}_c \,\boldsymbol{W}_{I,\mathrm{app}}^0}_{\boldsymbol{\tau}_{\mathrm{ext}}({{\boldsymbol{\mathcal{Q}}}})} \\
        & = \Bar{\mathbf{M}}({{\boldsymbol{\mathcal{Q}}}}) {{\dot{\boldsymbol{\nu}}}}+\boldsymbol{h}({{\boldsymbol{\mathcal{Q}}}}, {{\boldsymbol{\nu}}})+\boldsymbol{g}({{\boldsymbol{\mathcal{Q}}}}) + \boldsymbol{\tau}_{\mathrm{ext}}({{\boldsymbol{\mathcal{Q}}}}).
\end{aligned}
\label{EoM}
\end{equation}
with the mass matrix, Coriolis and centrifugal forces, gravitational vector, and external torque defined as \eqref{eqsprop}. The Coriolis and centrifugal vector $\boldsymbol{h}$ can be further factorized into an admissible Coriolis matrix $\mathbf{C}$ as shown in \eqref{Coroli}.

Higher order inverse dynamics can follow the same procedure. To illustrate, we further derive the 1-st order inverse dynamics in a closed form. Notice that by taking the first time derivative of \eqref{EoM}, it yields 
\begin{equation}
       \Bar{\mathbf{M}}({{\boldsymbol{\mathcal{Q}}}}){{\ddot{\boldsymbol{\nu}}}}+\dot{\Bar{\mathbf{M}}}({{\boldsymbol{\mathcal{Q}}}},{{\boldsymbol{\nu}}}) {{\dot{\boldsymbol{\nu}}}}+\dot{\boldsymbol{h}}({{\boldsymbol{\mathcal{Q}}}},{{\boldsymbol{\nu}}}, {{\dot{\boldsymbol{\nu}}}})+\dot{\boldsymbol{g}}({{\boldsymbol{\mathcal{Q}}}},{{\boldsymbol{\nu}}})+\dot{\boldsymbol{\tau}}_{\mathrm{ext}}({{\boldsymbol{\mathcal{Q}}}},{{\boldsymbol{\nu}}})=\dot{\Bar{\boldsymbol{\tau}}}.
       \nonumber
\end{equation}
Let $\dot{\Bar{\boldsymbol{h}}}({{\boldsymbol{\mathcal{Q}}}},{{\boldsymbol{\nu}}}, {{\dot{\boldsymbol{\nu}}}}):=\dot{\Bar{\mathbf{M}}}({{\boldsymbol{\mathcal{Q}}}},{{\boldsymbol{\nu}}}) {{\dot{\boldsymbol{\nu}}}}+\dot{\boldsymbol{h}}({{\boldsymbol{\mathcal{Q}}}},{{\boldsymbol{\nu}}}, {{\dot{\boldsymbol{\nu}}}}) $, hence
\begin{equation}
       \Bar{\mathbf{M}}({{\boldsymbol{\mathcal{Q}}}}){{\ddot{\boldsymbol{\nu}}}}+\dot{\Bar{\boldsymbol{h}}}({{\boldsymbol{\mathcal{Q}}}},{{\boldsymbol{\nu}}}, {{\dot{\boldsymbol{\nu}}}})+\dot{\boldsymbol{g}}({{\boldsymbol{\mathcal{Q}}}},{{\boldsymbol{\nu}}})+\dot{\boldsymbol{\tau}}_{\mathrm{ext}}({{\boldsymbol{\mathcal{Q}}}},{{\boldsymbol{\nu}}})=\dot{\Bar{\boldsymbol{\tau}}}.
   \label{1_st_EoM}
\end{equation}

The 3rd order recursion in kinematics Algorithm \ref{alg:kine} in a matrix form is 
\begin{align}
        \Ddot{\boldsymbol{V}}=& \mathbf{G}_p \mathbf{S} {{\ddot{\boldsymbol{\nu}}}} + 2 \mathbf{G}_p \boldsymbol{\mathrm{ad}}_{\mathbf{G}_p \mathbf{S} {{\boldsymbol{\nu}}}} \mathbf{S}\, {{\dot{\boldsymbol{\nu}}}}\nonumber\\
        &+ \mathbf{G}_p( \boldsymbol{\mathrm{ad}}_{\mathbf{G}_p \mathbf{S} {{\dot{\boldsymbol{\nu}}}}}+\boldsymbol{\mathrm{ad}}_{\mathbf{G}_p \boldsymbol{\mathrm{ad}}_{\mathbf{G}_p \mathbf{S} {{\boldsymbol{\nu}}}} \mathbf{S} {{\boldsymbol{\nu}}}}+\boldsymbol{\mathrm{ad}}^2_{\mathbf{G}_p \mathbf{S} {{\boldsymbol{\nu}}}})\mathbf{S} {{\boldsymbol{\nu}}},\\
        \Ddot{\boldsymbol{\Pi}}=&\mathbf{G}_c \mathbf{M}\Big(\Ddot{\boldsymbol{V}}-\boldsymbol{\mathrm{ad}}_{\mathbf{G}_p \mathbf{S} {{\boldsymbol{\nu}}}}(\mathbf{G}_p \mathbf{S} {{\dot{\boldsymbol{\nu}}}} + \mathbf{G}_p \boldsymbol{\mathrm{ad}}_{\mathbf{G}_p \mathbf{S} {{\boldsymbol{\nu}}}} \mathbf{S} {{\boldsymbol{\nu}}})\Big) \nonumber \\ 
        &-2 \mathbf{G}_c \boldsymbol{\mathrm{ad}}^T_{\boldsymbol{V}} \dot{\boldsymbol{\Pi}}- \mathbf{G}_c\big(\boldsymbol{\mathrm{ad}}_{\dot{\boldsymbol{V}}}+\boldsymbol{\mathrm{ad}}^2_{\boldsymbol{V}}\big)^T\boldsymbol{\Pi} \nonumber\\
        =&\mathbf{G}_c \mathbf{M}\Big(\mathbf{G}_p \mathbf{S} {{\ddot{\boldsymbol{\nu}}}} + 2 \mathbf{G}_p \boldsymbol{\mathrm{ad}}_{\mathbf{G}_p \mathbf{S} {{\boldsymbol{\nu}}}} \mathbf{S}\, {{\dot{\boldsymbol{\nu}}}} \nonumber\\
        &+ \mathbf{G}_p( \boldsymbol{\mathrm{ad}}_{\mathbf{G}_p \mathbf{S} {{\dot{\boldsymbol{\nu}}}}}+\boldsymbol{\mathrm{ad}}_{\mathbf{G}_p \boldsymbol{\mathrm{ad}}_{\mathbf{G}_p \mathbf{S} {{\boldsymbol{\nu}}}} \mathbf{S} {{\boldsymbol{\nu}}}}+\boldsymbol{\mathrm{ad}}^2_{\mathbf{G}_p \mathbf{S} {{\boldsymbol{\nu}}}})\mathbf{S} {{\boldsymbol{\nu}}} \nonumber\\
        &-\boldsymbol{\mathrm{ad}}_{\mathbf{G}_p \mathbf{S} {{\boldsymbol{\nu}}}}(\mathbf{G}_p \mathbf{S} {{\dot{\boldsymbol{\nu}}}} + \mathbf{G}_p \boldsymbol{\mathrm{ad}}_{\mathbf{G}_p \mathbf{S} {{\boldsymbol{\nu}}}} \mathbf{S} {{\boldsymbol{\nu}}})\Big) \nonumber\\
        &-2 \mathbf{G}_c \boldsymbol{\mathrm{ad}}^T_{\mathbf{G}_p \mathbf{S} {{\boldsymbol{\nu}}}} \Big( \mathbf{M} \mathbf{G}_p \mathbf{S} {{\dot{\boldsymbol{\nu}}}} + \mathbf{M} \mathbf{G}_p \boldsymbol{\mathrm{ad}}_{\mathbf{G}_p \mathbf{S} {{\boldsymbol{\nu}}}} \mathbf{S} {{\boldsymbol{\nu}}}  -  \boldsymbol{\mathrm{ad}}^T_{\mathbf{G}_p \mathbf{S} {{\boldsymbol{\nu}}}} \mathbf{M} \mathbf{G}_p \mathbf{S} {{\boldsymbol{\nu}}} \Big)  \nonumber \\
        &- \mathbf{G}_c \big( \boldsymbol{\mathrm{ad}}_{\mathbf{G}_p \mathbf{S} {{\dot{\boldsymbol{\nu}}}}}+\boldsymbol{\mathrm{ad}}_{\mathbf{G}_p \boldsymbol{\mathrm{ad}}_{\mathbf{G}_p \mathbf{S} {{\boldsymbol{\nu}}}} \mathbf{S} {{\boldsymbol{\nu}}}}+\boldsymbol{\mathrm{ad}}^2_{\mathbf{G}_p \mathbf{S} {{\boldsymbol{\nu}}}}\big)^T \mathbf{M} \mathbf{G}_p \mathbf{S} {{\boldsymbol{\nu}}}.
\end{align}

Since the expressions for the 1-st order gravitational and external wrenches are given by, using \eqref{extweriv} and \eqref{gravweriv}, 
\begin{align}
        \dot{\boldsymbol{F}}_{\mathrm{app}}&= - \mathbf{G}_c \,\dot{\boldsymbol{W}}_{I,\mathrm{app}}^0 = \mathbf{G}_c \boldsymbol{\mathrm{ad}}^T_{\mathbf{G}_p \mathbf{S} {{\boldsymbol{\nu}}}} \boldsymbol{W}_{I,\mathrm{app}}^0 - \mathbf{G}_c \boldsymbol{\mathrm{Ad}}^T_{\mathbf{C}^{-1}} \dot{\boldsymbol{W}}_{I,\mathrm{app}}^b,\\
        \dot{\boldsymbol{F}}_{\mathrm{grav}}&= - \mathbf{G}_c \,\dot{\boldsymbol{W}}_{I,\mathrm{grav}}^0 = \mathbf{G}_c(\mathbf{M} \boldsymbol{\mathrm{ad}}_{\mathbf{G}_p \mathbf{S} {{\boldsymbol{\nu}}}}+\boldsymbol{\mathrm{ad}}^T_{\mathbf{G}_p \mathbf{S} {{\boldsymbol{\nu}}}} \mathbf{M}) \boldsymbol{G}^0_I.
\end{align}
Thus, the 1-st order inverse dynamics admits the closed form
\begin{equation}
\begin{split}
         \boldsymbol{W}&=(\dot{\boldsymbol{\Pi}} + \boldsymbol{F}_{\mathrm{app}} + \boldsymbol{F}_{\mathrm{grav}}),\\
         \dot{\boldsymbol{W}}&=(\Ddot{\boldsymbol{\Pi}} + \dot{\boldsymbol{F}}_{\mathrm{app}} + \dot{\boldsymbol{F}}_{\mathrm{grav}}),\\
        \dot{\Bar{\boldsymbol{\tau}}} &= \dot{\mathbf{S}}^T\boldsymbol{W}+\mathbf{S}^T\dot{\boldsymbol{W}}\\
        &= \mathbf{S}^T\boldsymbol{\mathrm{ad}}^T_{\mathbf{G}_p \mathbf{S} {{\boldsymbol{\nu}}}} \big(\mathbf{G}_c \mathbf{M} \mathbf{G}_p \mathbf{S} {{\dot{\boldsymbol{\nu}}}} + \mathbf{G}_c \mathbf{M} \mathbf{G}_p \boldsymbol{\mathrm{ad}}_{\mathbf{G}_p \mathbf{S} {{\boldsymbol{\nu}}}} \mathbf{S} {{\boldsymbol{\nu}}} \\
        & - \mathbf{G}_c \boldsymbol{\mathrm{ad}}^T_{\mathbf{G}_p \mathbf{S} {{\boldsymbol{\nu}}}} \mathbf{M} \mathbf{G}_p \mathbf{S} {{\boldsymbol{\nu}}} - \mathbf{G}_c \mathbf{M} \boldsymbol{G}^0_I - \mathbf{G}_c \boldsymbol{\mathrm{Ad}}^T_{\mathbf{C}^{-1}} \boldsymbol{W}_{I,\mathrm{app}}^b \big)\\
        &+ \mathbf{S}^T \mathbf{G}_c \mathbf{M} \mathbf{G}_p \mathbf{S} {{\ddot{\boldsymbol{\nu}}}}+ \mathbf{S}^T\Tilde{\Ddot{\boldsymbol{\Pi}}}({{\boldsymbol{\mathcal{Q}}}},{{\boldsymbol{\nu}}},{{\dot{\boldsymbol{\nu}}}}) \\
        &+ \mathbf{S}^T \mathbf{G}_c\boldsymbol{\mathrm{ad}}^T_{\mathbf{G}_p \mathbf{S} {{\boldsymbol{\nu}}}} \boldsymbol{\mathrm{Ad}}^T_{\mathbf{C}^{-1}} \boldsymbol{W}_{I,\mathrm{app}}^b-\mathbf{S}^T \mathbf{G}_c\boldsymbol{\mathrm{Ad}}^T_{\mathbf{C}^{-1}} \dot{\boldsymbol{W}}_{I,\mathrm{app}}^b\\
        &+\mathbf{S}^T \mathbf{G}_c(\mathbf{M} \boldsymbol{\mathrm{ad}}_{\mathbf{G}_p \mathbf{S} {{\boldsymbol{\nu}}}}+\boldsymbol{\mathrm{ad}}^T_{\mathbf{G}_p \mathbf{S} {{\boldsymbol{\nu}}}} \mathbf{M}) \boldsymbol{G}^0_I\\
        &= \Bar{\mathbf{M}}({{\boldsymbol{\mathcal{Q}}}}){{\ddot{\boldsymbol{\nu}}}}+\dot{\Bar{\boldsymbol{h}}}({{\boldsymbol{\mathcal{Q}}}},{{\boldsymbol{\nu}}}, {{\dot{\boldsymbol{\nu}}}})+\dot{\boldsymbol{g}}({{\boldsymbol{\mathcal{Q}}}},{{\boldsymbol{\nu}}})+\dot{\boldsymbol{\tau}}_{\mathrm{ext}}({{\boldsymbol{\mathcal{Q}}}},{{\boldsymbol{\nu}}}),
\end{split}
\end{equation}
where $\Tilde{\Ddot{\boldsymbol{\Pi}}}({{\boldsymbol{\mathcal{Q}}}},{{\boldsymbol{\nu}}},{{\dot{\boldsymbol{\nu}}}}) :=\Ddot{\boldsymbol{\Pi}}-\mathbf{G}_c \mathbf{M} \mathbf{G}_p \mathbf{S} {{\ddot{\boldsymbol{\nu}}}}$. Hence, the terms in \eqref{1_st_EoM} can be explicitly expanded as in \eqref{eqprop1}.
\section{{{Proof of Lemma 2}}}
\label{append:proof_Lemma2}
We start from the equations of motion for a single isolated rigid body $B_i$, which can be recast as
\begin{equation}
     \boldsymbol{W}_i = \mathbf{M}^0_i \dot{\boldsymbol{V}}_i^0 - \boldsymbol{\mathrm{ad}}^T_{{\boldsymbol{V}}_i^0} \mathbf{M}^0_i{\boldsymbol{V}}_i^0-\boldsymbol{W}_{\mathrm{app},i}^0-\boldsymbol{W}_{\mathrm{grav},i}^0
     =\mathbf{M}_i^A \dot{\boldsymbol{V}}_i^0+ \boldsymbol{W}_i^A, 
     \label{arti}
\end{equation}
where $\mathbf{M}_i^A=\mathbf{M}^0_i$ and $\boldsymbol{W}_i^A=- \boldsymbol{\mathrm{ad}}^T_{{\boldsymbol{V}}_i^0} \mathbf{M}^0_i{\boldsymbol{V}}_i^0-\boldsymbol{W}_{\mathrm{app},i}^0-\boldsymbol{W}_{\mathrm{grav},i}^0$.  

Now, consider connecting two isolated rigid bodies $B_1$ and $B_2$ through a joint $j_2$ such that $B_2$ is a child of $B_1$. The combined body, called an articulated body and labeled $A_1$, has $B_1$ as its handle. The algorithm recursively solves for the joint acceleration $\Ddot{q}_2$ and base acceleration $\dot{\boldsymbol{V}}_{\mathrm{base}}^0$, given the state, the joint torque $\Bar{\boldsymbol{\tau}}_2$ and propeller wrench at the base $\boldsymbol{W}^0_{1,\mathrm{prop}}$.  

After connection through $j_2$, $\boldsymbol{W}_2$ becomes
\begin{equation}
\begin{split}
\boldsymbol{W}_2 &= \mathbf{M}^0_2 \dot{\boldsymbol{V}}_1^0+ \mathbf{M}^0_2(\boldsymbol{S}_2\Ddot{q}_2+\dot{\boldsymbol{S}}_2\dot{q}_2) - \boldsymbol{\mathrm{ad}}^T_{{\boldsymbol{V}}_2^0} \mathbf{M}^0_2{\boldsymbol{V}}_2^0-\boldsymbol{W}_{\mathrm{app},2}^0-\boldsymbol{W}_{\mathrm{grav},2}^0.
\end{split}
\end{equation}
Projecting this wrench onto the screw axis $\boldsymbol{S}_2$ yields
\begin{equation}
\begin{split}
        \Bar{\boldsymbol{\tau}}_2 &= \boldsymbol{S}_2^T \boldsymbol{W}_2 = \boldsymbol{S}_2^T \mathbf{M}^0_2 \dot{\boldsymbol{V}}_1^0+ \boldsymbol{S}_2^T \mathbf{M}^0_2 \boldsymbol{S}_2\Ddot{q}_2+\boldsymbol{S}_2^T \mathbf{M}^0_2 \dot{\boldsymbol{S}}_2\dot{q}_2 \\
       &- \boldsymbol{S}_2^T \boldsymbol{\mathrm{ad}}^T_{{\boldsymbol{V}}_2^0} \mathbf{M}^0_2{\boldsymbol{V}}_2^0-\boldsymbol{S}_2^T \boldsymbol{W}_{\mathrm{app},2}^0-\boldsymbol{S}_2^T \boldsymbol{W}_{\mathrm{grav},2}^0.
\end{split}
\end{equation}
Rearranging, we extract $\Ddot{q}_2$ as
\begin{equation}
\begin{split}
       \Ddot{q}_2  = &\big(\boldsymbol{S}_2^T \mathbf{M}^0_2 \boldsymbol{S}_2\big)^{-1} \bigg(\Bar{\boldsymbol{\tau}}_2 - \boldsymbol{S}_2^T\big( \mathbf{M}^0_2 \dot{\boldsymbol{V}}_1^0 +\mathbf{M}^0_2 \dot{\boldsymbol{S}}_2\dot{q}_2\\
       &- \boldsymbol{\mathrm{ad}}^T_{{\boldsymbol{V}}_2^0} \mathbf{M}^0_2{\boldsymbol{V}}_2^0- \boldsymbol{W}_{\mathrm{app},2}^0- \boldsymbol{W}_{\mathrm{grav},2}^0\big)\bigg) \\
       =&- \big(\boldsymbol{S}_2^T \mathbf{M}^0_2 \boldsymbol{S}_2\big)^{-1}  \boldsymbol{S}_2^T \mathbf{M}^0_2 \dot{\boldsymbol{V}}_1^0 +\Ddot{\Tilde{q}}_2,
\end{split}
\label{accq2}
\end{equation}
where $\Ddot{\Tilde{q}}_2(\boldsymbol{q}_1,\boldsymbol{q}_2,\boldsymbol{V}_1^0,\dot{q}_2) :=\Ddot{q}_2 + \big(\boldsymbol{S}_2^T \mathbf{M}^0_2 \boldsymbol{S}_2\big)^{-1}  \boldsymbol{S}_2^T \mathbf{M}^0_2 \dot{\boldsymbol{V}}_1^0$. 
Note that the product $\big(\boldsymbol{S}_2^T \mathbf{M}^0_2 \boldsymbol{S}_2\big)$ is a non-zero scalar when multi-DoF joints are treated as an arrangement of 1-DoF joints. Hence, the new value of the $\boldsymbol{W}_1$, denoted by $\Bar{\boldsymbol{W}}_1$, has to account for the wrench $\boldsymbol{W}_2$ coming from the child body $B_2$ as follows
\begin{equation}
\begin{split}
\Bar{\boldsymbol{W}}_1&= {\boldsymbol{W}}_1+{\boldsymbol{W}}_2\\
&= \big((\mathbf{M}^0_1 +  \mathbf{M}^0_2)-\mathbf{M}^0_2 \boldsymbol{S}_2 (\boldsymbol{S}_2^T \mathbf{M}^0_2 \boldsymbol{S}_2)^{-1}  \boldsymbol{S}_2^T \mathbf{M}^0_2\big)\dot{\boldsymbol{V}}_1^0  \\
&+\mathbf{M}^0_2  (\boldsymbol{S}_2\Ddot{\Tilde{q}}_2+\dot{\boldsymbol{S}}_2\dot{q}_2) 
- \boldsymbol{\mathrm{ad}}^T_{{\boldsymbol{V}}_1^0} \mathbf{M}^0_1{\boldsymbol{V}}_1^0 
- \boldsymbol{\mathrm{ad}}^T_{{\boldsymbol{V}}_2^0} \mathbf{M}^0_2{\boldsymbol{V}}_2^0\\
& -\boldsymbol{W}_{\mathrm{app},1}^0-\boldsymbol{W}_{\mathrm{grav},1}^0 
-\boldsymbol{W}_{\mathrm{app},2}^0-\boldsymbol{W}_{\mathrm{grav},2}^0
= \Bar{\mathbf{M}}_1^A \dot{\boldsymbol{V}}_1^0+ \Bar{\boldsymbol{W}}_1^A.
\end{split}
\label{handleB1}
\end{equation}

where the articulated inertia $\mathbf{M}_1^A$ and bias $\boldsymbol{W}_1^A$ change to these new values $\Bar{\mathbf{M}}_1^A$ and $\Bar{\boldsymbol{W}}_1^A$, respectively
\begin{equation}
\begin{split}
    \Bar{\mathbf{M}}_1^A=&\big((\mathbf{M}^0_1 +  \mathbf{M}^0_2)-\mathbf{M}^0_2 \boldsymbol{S}_2 (\boldsymbol{S}_2^T \mathbf{M}^0_2 \boldsymbol{S}_2)^{-1}  \boldsymbol{S}_2^T \mathbf{M}^0_2\big),\\ 
    \Bar{\boldsymbol{W}}_1^A =&\mathbf{M}^0_2  (\boldsymbol{S}_2\Ddot{\Tilde{q}}_2+\dot{\boldsymbol{S}}_2\dot{q}_2) 
    - \boldsymbol{\mathrm{ad}}^T_{{\boldsymbol{V}}_1^0} \mathbf{M}^0_1{\boldsymbol{V}}_1^0 
    - \boldsymbol{\mathrm{ad}}^T_{{\boldsymbol{V}}_2^0} \mathbf{M}^0_2{\boldsymbol{V}}_2^0\\
    & -\boldsymbol{W}_{\mathrm{app},1}^0-\boldsymbol{W}_{\mathrm{grav},1}^0 
    -\boldsymbol{W}_{\mathrm{app},2}^0-\boldsymbol{W}_{\mathrm{grav},2}^0.
\end{split}
\end{equation}

The articulated body $A_1$ now has a handle $B_1$ with updated equations of motion \eqref{handleB1}. Body $A_1$ can be connected to another body as its child or its parent through a joint between the handle $B_1$ and the new body. After reaching the root body of the tree, i.e., the floating base, the algorithm concludes this backward recursion with the articulated inertia of the root body $\mathbf{M}_{\mathrm{base}}^A$ accumulating all previous inertias, and its bias wrench $\boldsymbol{W}_{\mathrm{base}}^A$ encompassing all other bodies' wrenches. 

A forward recursion is then initiated by computing the twist of the base $\dot{\boldsymbol{V}}_{\mathrm{base}}^0$ from this system of six linear equations using $LDL^T$ decomposition, since the articulated inertia $\mathbf{M}_{\mathrm{base}}^A$ is symmetric positive-definite \cite{featherstone2014rigid},
\begin{equation}
    \boldsymbol{W}^0_{1,\mathrm{prop}}= \mathbf{M}_{\mathrm{base}}^A \dot{\boldsymbol{V}}_{\mathrm{base}}^0+ \boldsymbol{W}_{\mathrm{base}}^A.
    \label{baseacc}
\end{equation}
The result is then propagated to the children in order to solve for the joint accelerations using \eqref{accq2}. The children twists are then updated, which will be needed to re-evaluate \eqref{accq2} in the next iteration. 

Suppose now that it is desired to recursively calculate the 1-st time derivative of the joint and base accelerations, i.e., the joint $\dddot{q}_2$ and base $\Ddot{\boldsymbol{V}}_{\mathrm{base}}^0$ jerks, given the state, $\Bar{\boldsymbol{\tau}}_2$ and its time derivative $\dot{\Bar{\boldsymbol{\tau}}}_2$, $\boldsymbol{W}^0_{1,\mathrm{prop}}$ and its derivative $\dot{\boldsymbol{W}}^0_{1,\mathrm{prop}}$, along with $\Ddot{q}_2$ and $\dot{\boldsymbol{V}}_{\mathrm{base}}^0$ obtained from \eqref{accq2} and \eqref{baseacc}.  

The 1-st time derivative of the constraint wrench between $B_1$ and $B_2$, $\dot{\boldsymbol{W}}_2$, is given by
\begin{equation}
\begin{split}
\dot{\boldsymbol{W}}_2 &= \Ddot{\boldsymbol{\Pi}}^0_2-\dot{\boldsymbol{W}}_{\mathrm{app},2}^0-\dot{\boldsymbol{W}}_{\mathrm{grav},2}^0\\
&= \mathbf{M}_2^0\!\left( \Ddot{\Tilde{\boldsymbol{V}}}_2^0 + \Ddot{\boldsymbol{V}}_1^0+\boldsymbol{S}_2 \dddot{q}_2 - \boldsymbol{\mathrm{ad}}_{\boldsymbol{V}_2^0}\dot{\boldsymbol{V}}_2^0\right) \\
&-2  \boldsymbol{\mathrm{ad}}_{{\boldsymbol{V}}_2^0}^T (
         \mathbf{M}_2^0 \dot{\boldsymbol{V}}_2^0 -  \boldsymbol{\mathrm{ad}}_{\boldsymbol{V}_2^0}^T \mathbf{M}_2^0 \boldsymbol{V}_2^0 )\\
&- \left(\boldsymbol{\mathrm{ad}}_{\dot{\boldsymbol{V}}^0_2}+\boldsymbol{\mathrm{ad}}^2_{\boldsymbol{V}^0_2}\right)^T\mathbf{M}^0_2\boldsymbol{V}^0_2 
-\dot{\boldsymbol{W}}_{\mathrm{app},2}^0-\dot{\boldsymbol{W}}_{\mathrm{grav},2}^0\\
&=  \mathbf{M}_2^0 \Ddot{\Tilde{\boldsymbol{V}}}_2^0+ \mathbf{M}_2^0 \Ddot{\boldsymbol{V}}_1^0+\mathbf{M}_2^0 \boldsymbol{S}_2 \dddot{q}_2 +\Ddot{\Tilde{\boldsymbol{\Pi}}}^0_2 
-\dot{\boldsymbol{W}}_{\mathrm{app},2}^0-\dot{\boldsymbol{W}}_{\mathrm{grav},2}^0,
\end{split}
\end{equation}
where $\Ddot{\Tilde{\boldsymbol{V}}}_2^0$ and $\Ddot{\Tilde{\boldsymbol{\Pi}}}_2^0$ are defined as
\begin{equation}
        \Ddot{\Tilde{\boldsymbol{V}}}_2^0 := 2 \dot{\boldsymbol{S}}_2 \Ddot{q}_2 + \Ddot{\boldsymbol{S}}_2 \dot{q}_2,\qquad
        \Ddot{\Tilde{\boldsymbol{\Pi}}}_2^0 := \Ddot{\boldsymbol{\Pi}}_2^0 - \mathbf{M}_2^0\Ddot{\boldsymbol{V}}_2^0.
\end{equation}

Hence, the time derivative of the joint torque is expressed as
\begin{equation}
\begin{split}
        \dot{\Bar{\boldsymbol{\tau}}}_2 &= \dot{\boldsymbol{S}}_2^T \boldsymbol{W}_2 + {\boldsymbol{S}}_2^T \dot{\boldsymbol{W}}_2 \\
        &=\dot{\Tilde{\boldsymbol{\tau}}}_2+ \boldsymbol{S}_2^T \mathbf{M}^0_2 \Ddot{\boldsymbol{V}}_1^0+ \boldsymbol{S}_2^T \mathbf{M}^0_2 \boldsymbol{S}_2\dddot{q}_2+\boldsymbol{S}_2^T \mathbf{M}^0_2 \Ddot{\Tilde{\boldsymbol{V}}}_2^0 \\
        &+  \boldsymbol{S}_2^T \Ddot{\Tilde{\boldsymbol{\Pi}}}_2^0
        -\boldsymbol{S}_2^T \dot{\boldsymbol{W}}_{\mathrm{app},2}^0
        -\boldsymbol{S}_2^T \dot{\boldsymbol{W}}_{\mathrm{grav},2}^0,
\end{split}
\end{equation}
where $\dot{\Tilde{\boldsymbol{\tau}}}_2 :=  \dot{\Bar{\boldsymbol{\tau}}}_2 - {\boldsymbol{S}}_2^T \dot{\boldsymbol{W}}_2$.  

This yields the equation for $\dddot{q}_2$ as
\begin{equation}
\begin{split}
       \dddot{q}_2  = &\big(\boldsymbol{S}_2^T \mathbf{M}^0_2 \boldsymbol{S}_2\big)^{-1} \bigg(\dot{\Bar{\boldsymbol{\tau}}}_2-\dot{\Tilde{\boldsymbol{\tau}}}_2 - \boldsymbol{S}_2^T\big( \mathbf{M}^0_2 \Ddot{\boldsymbol{V}}_1^0 +\mathbf{M}^0_2 \Ddot{\Tilde{\boldsymbol{V}}}_2^0 \\
       &+ \Ddot{\Tilde{\boldsymbol{\Pi}}}_2^0- \dot{\boldsymbol{W}}_{\mathrm{app},2}^0- \dot{\boldsymbol{W}}_{\mathrm{grav},2}^0\big)\bigg) \\
       =&- \big(\boldsymbol{S}_2^T \mathbf{M}^0_2 \boldsymbol{S}_2\big)^{-1}  \boldsymbol{S}_2^T \mathbf{M}^0_2 \Ddot{\boldsymbol{V}}_1^0 +\dddot{\Tilde{q}}_2,
\end{split}
\label{acccq2}
\end{equation}
where $\dddot{\Tilde{q}}_2 :=\dddot{q}_2 + \big(\boldsymbol{S}_2^T \mathbf{M}^0_2 \boldsymbol{S}_2\big)^{-1}  \boldsymbol{S}_2^T \mathbf{M}^0_2 \Ddot{\boldsymbol{V}}_1^0$.  

When connected to $B_1$, the time derivative of the updated constraint wrench transmitted to its parent $\dot{\Bar{\boldsymbol{W}}}_1$ becomes
\begin{equation}
\begin{split}
\dot{\Bar{\boldsymbol{W}}}_1&= \dot{{\boldsymbol{W}}}_1+\dot{{\boldsymbol{W}}}_2\\
&= \big((\mathbf{M}^0_1 +  \mathbf{M}^0_2)-\mathbf{M}^0_2 \boldsymbol{S}_2 (\boldsymbol{S}_2^T \mathbf{M}^0_2 \boldsymbol{S}_2)^{-1}  \boldsymbol{S}_2^T \mathbf{M}^0_2\big)\Ddot{\boldsymbol{V}}_1^0 \\
&+\mathbf{M}_2^0 \Ddot{\Tilde{\boldsymbol{V}}}_2^0+\mathbf{M}_2^0 \boldsymbol{S}_2 \dddot{\Tilde{q}}_2 +\Ddot{\Tilde{\boldsymbol{\Pi}}}_1^0+\Ddot{\Tilde{\boldsymbol{\Pi}}}_2^0 \\
&-\dot{\boldsymbol{W}}_{\mathrm{app},1}^0-\dot{\boldsymbol{W}}_{\mathrm{grav},1}^0 
-\dot{\boldsymbol{W}}_{\mathrm{app},2}^0-\dot{\boldsymbol{W}}_{\mathrm{grav},2}^0
= {\Bar{\mathbf{M}}}_1^A \Ddot{\boldsymbol{V}}_1^0+ \dot{\Bar{\boldsymbol{W}}}_1^A,
\end{split}
\end{equation}
where $\Ddot{\Tilde{\boldsymbol{\Pi}}}_1^0 := \Ddot{\boldsymbol{\Pi}}_1^0 - \mathbf{M}_1^0\Ddot{\boldsymbol{V}}_1^0$.  
with the property ${\Bar{\mathbf{M}}}_1^A$ being the same as the acceleration level ($r=0$).  Hence, it can be calculated once and reused for any order.

The procedure continues by connecting $B_1$ as a child to its predecessor in the chain until the base body is reached. Then, the following system of six linear equations is solved for $\Ddot{\boldsymbol{V}}_{\mathrm{base}}^0$ using $LDL^T$ decomposition:
\begin{equation}
    \dot{\boldsymbol{W}}^0_{1,\mathrm{prop}}= \mathbf{M}_{\mathrm{base}}^A \Ddot{\boldsymbol{V}}_{\mathrm{base}}^0+ \dot{\boldsymbol{W}}_{\mathrm{base}}^A.
\end{equation}

Afterwards, the forward recursion begins by propagating the base jerk $\Ddot{\boldsymbol{V}}_{\mathrm{base}}^0$ to its children to compute their joint jerks from \eqref{acccq2} and the corresponding link jerks. It proceeds until all terminal nodes in each branch are reached. 

The steps for $r\geq 2$ follow the exact same pattern by induction, which concludes the proof.
\section{{{Initialization procedure}}}
To obtain the kinematics and its higher orders, the robot has to be set up as following so that its model parameters can become available to Algo. \ref{alg:kine}.
\label{append:initial}
\begin{itemize}
    \item \textbf{Step 1.} Assign a reference frame $\mathcal{F}_{i}$ to each body $i$ in the tree, preferably with the origin at CoM, and denote the inertial frame by $\mathcal{F}_{0}$ and the base frame by $\mathcal{F}_{1}$. Hence, $\mathbf{C}^{j}_{i} \in SE(3)$ is a function of all the configuration variables in the path connecting $i$ to $j$. 
    \item \textbf{Step 2.} Put the system at home configuration $\mathcal{H}_{0}$, i.e. a configuration where all joint variables $q_i(0) \equiv 0, \, \text{and }  \mathbf{C}^{0}_{1}(0)=\mathbf{I}_{4\times4}$, which can be chosen arbitrary. Without loss of generality, this assumes the inertial frame is coincident with the floating-base fixed frame at such configuration.
    \item \textbf{Step 3.} In the configuration $\mathcal{H}_{0}$, the spatial screws list $\{\boldsymbol{Y}_1, \boldsymbol{Y}_2,..., \boldsymbol{Y}_n\}$, encoding geometric information about the joint models \cite{featherstone2014rigid,muller2018screw}, is computed from \eqref{constantScrew} for all the $n$ 1-DoF joints in the tree. Spatial representation is chosen here since it is right-invariant, meaning that the spatial twist of a body is invariant under a change of body-fixed frame \cite{murray2017mathematical}. If the joint is not lower pair \cite{selig2007geometric}, like universal joint, it may be represented by an arrangement of 1-DoF joints. For example, the 2 DoF universal joint can be modeled by 2 consecutive 1-DoF revolute joints $\boldsymbol{q} \in \mathbb{T}^2\equiv \mathbb{S}^1 \times\mathbb{S}^1$ such that their axes are always mutually-orthogonal. 
    
    However, if the joint has a motion subgroup which is a generator of subgroup in $SE(3)$ and not isomorphic to $\mathbb{T}^{n_1} \times \mathbb{R}^{n_2}$, i.e. spherical joint $SO(3)$, the relative configuration may be better described using the corresponding group element instead of the minimal coordinates to avoid parameterization singularities. For instance, let $\boldsymbol{q}\in SO(3)$ be a spherical joint between bodies $j$ and $i$, then $\mathbf{C}^{j}_{i}(\boldsymbol{q}(t))= \left(\begin{smallmatrix}
        \boldsymbol{q}(t) & \boldsymbol{x} \\ \mathbf{0} & 1
    \end{smallmatrix}\right)$, where $\boldsymbol{x} \in \mathbb{R}^3$ is the constant position of  $\mathcal{F}_{i}$'s origin expressed in $\mathcal{F}_{j}$. Here, $\boldsymbol{q}\in SO(3)$ is preferred over $\boldsymbol{q}\in \mathbb{T}^{3}$ since they are not isomorphic. Contrarily to Euler Angles parameterization, this characterization of the spherical joint is singularity-free. In a floating-base system, the base, indexed here by 1, is connected to the ground 0, i.e. inertial frame, by a free motion joint whose motion subgroup is the entire $SE(3)$. Hence its relative configuration to the inertial frame is taken as $\mathbf{C}^{0}_{1}(t) \in SE(3)$.
    \item \textbf{Step 4.} Create 2 lists $c\{i\}$ and $p\{i\}$ for each body $i$ to store the indexes of children and parent of body $i$. In the directed graph Fig.~\ref{kinematicTree}, with in-degree $= 1$ and out-degree $\geq1$, a child of body $i$ is any body (node) $j_c \in c(i)$ connected to body $i$ by a joint (edge) $j_c$, whereas a parent of body $i$ is \textit{the} body $j_p \in p(i)$ whose connection to body $i$ is the joint $i$. 
    \item \textbf{Step 5.} Let $\mathbf{A}^{p\{i\}}_{i}$ be the relative constant configuration of each body $\mathcal{F}_{i}$ to its parent frame $\mathcal{F}_{p\{i\}}$ in the configuration $\mathcal{H}_{0}$. Denote by $\boldsymbol{S}_i$ the instantaneous spatial screw vector of joint $i$ expressed in the inertial frame at the current configuration.
\end{itemize} 

\section{{{Derivations For H-GRNE-FBS and H-GABI-FBS}}}
\label{append:Derivations}
This section collects analytical expressions for the time derivatives of quantities used in both algorithms \ref{alg:Dyn} and \ref{alg:ForwardDyn}. According to Newton's second law, body $j$ momentum time rate of change is equal to the total wrench acting on this body
\begin{equation}
    \dot{\boldsymbol{\Pi}}^0_j= (\boldsymbol{W}^0_j)^{\mathrm{tot}}
\end{equation}
where $(\boldsymbol{W}^0_j)^{\mathrm{tot}}$ is the total spatial wrench applied at body $j$ which can be calculated from
\begin{equation}
    (\boldsymbol{W}^0_j)^{\mathrm{tot}}= \boldsymbol{W}^0_{j,\mathrm{app}}+\boldsymbol{W}^0_{j,\mathrm{grav}} + \boldsymbol{W}^0_{p(j)}  - \sum_{c(j)} \boldsymbol{W}^0_{c(j)}
\end{equation}

The spatial momentum $\boldsymbol{\Pi}^0_j$ and its higher-order time derivatives can be evaluated from kinematical quantities, computed in Algorithm (\ref{alg:kine}), using Euler–Poincaré equation $\dot{\boldsymbol{\Pi}}^0_j=\mathbf{M}^0_j \dot{\boldsymbol{V}}^0_j-\boldsymbol{\mathrm{ad}}_{\boldsymbol{V}^0_j}\boldsymbol{\Pi}^0_j$ and its derivatives \cite{marsden2013introduction}. In the 0-th order it is readily given by 
\begin{equation}
    \boldsymbol{\Pi}^0_j= \mathbf{M}^0_j \boldsymbol{V}^0_j.
\end{equation}
where the spatial mass matrix $\mathbf{M}^0_j$ is related to the body-fixed constant mass matrix $\mathbf{M}^b_j$ by this frame transformation \cite{murray2017mathematical} 
\begin{equation}
    \mathbf{M}^0_j= \boldsymbol{\mathrm{Ad}}_{\mathbf{C}^0_j}^{-T} \mathbf{M}^b_j \boldsymbol{\mathrm{Ad}}_{\mathbf{C}^0_j}^{-1}
    \label{spaM}
\end{equation}
The time derivatives of the matrix $\mathbf{M}^0_j$ can be expressed in terms of the transformation matrix and the higher-order spatial twists of the body. In general, its $r$-th time derivative has the form, after some algebraic manipulation using \eqref{eqFormulas} and dropping the indexes for clarity
{{\begin{equation}
\resizebox{\columnwidth}{!}{$
\mathbf{M}^{(r)}=
\begin{cases}
\mathbf{M}, & r=0,\\[3pt]
\displaystyle
- \sum_{k=0}^{r-1} \binom{r-1}{k}
\left( \mathbf{M}^{(r-1-k)} \boldsymbol{\mathrm{ad}}_{\boldsymbol{V}^{(k)}}
+ \boldsymbol{\mathrm{ad}}_{\boldsymbol{V}^{(k)}}^T \mathbf{M}^{(r-1-k)} \right), & r\ge 1.
\end{cases}
$}
\label{Mass}
\end{equation}}}
Now, the $r$-th time derivative of body's $j$ momentum is obtained from this relation
\begin{equation}
    \boldsymbol{\Pi}^{(r)} = \sum_{k=0}^{r} \binom{r}{k} \mathbf{M}^{(r-k)} \boldsymbol{V}^{(k)}
    \label{Momu}
\end{equation}
Also, the spatial gravity wrench on body $j$ can be calculated from the spatial mass matrix \eqref{spaM} as
\begin{equation}
\begin{split}
       &\boldsymbol{W}^0_{j,\mathrm{grav}}=\mathbf{M}^0_j \, \boldsymbol{G}^0,\,\,
       \boldsymbol{G}^0=\left(\begin{smallmatrix}
           0&0&0&0&0&-g
       \end{smallmatrix}\right)^T
\end{split}
\end{equation}
where $g$ is the gravity acceleration constant. It is easily seen that its higher order time derivatives are a function of \eqref{Mass} and given by
\begin{equation}
   \boldsymbol{W}^{(r)}= \mathbf{M}^{(r)} \boldsymbol{G}^0
    \label{gravweriv}
\end{equation}

The bias momentum $\boldsymbol{\Pi}_{j,\mathrm{bias}}^{(r)}$ is given by 
\begin{equation}
\boldsymbol{\Pi}_{j,\mathrm{bias}}^{(r)}
=
\sum_{k=0}^{r-1}
\binom{r}{k}
(\mathbf{M}_{j}^{0})^{(r-k)}
(\boldsymbol{V}_{j}^{0})^{(k)}
\end{equation} 
whereas the bias spatial twist of body $i$, denoted by $(\boldsymbol{V}_{i,\mathrm{bias}}^0)^{(r)}$, is calculated from 
\begin{equation}
    (\boldsymbol{V}_{i,\mathrm{bias}}^0)^{(r)}
=
\sum_{k=1}^{r}
\binom{r}{k}
\boldsymbol{S}_{i}^{(k)}
\boldsymbol{q}_{i}^{(r-k+1)}
\end{equation}

Let $\boldsymbol{W}^b_{j}$ denote the applied wrench at body $j$ expressed in its body-fixed frame $\mathcal{F}_{j}$, then the mapping to the spatial wrench is \cite{lynch2017modern}
\begin{equation}
     \boldsymbol{W}^0_{j}= \boldsymbol{\mathrm{Ad}}_{(\mathbf{C}^0_j)^{-1}}^T \boldsymbol{W}^b_{j}.
\end{equation}
Hence, its first time derivative is, after applying \eqref{Adjointderiv},
\begin{equation}
  \dot{\boldsymbol{W}}^0_{j}= -\boldsymbol{\mathrm{ad}}^T_{\boldsymbol{V}^0}\,\boldsymbol{\mathrm{Ad}}_{\mathbf{C}^{-1}}^T \boldsymbol{W}^b_{j}+  \boldsymbol{\mathrm{Ad}}_{\mathbf{C}^{-1}}^T \dot{\boldsymbol{W}}^b_{j}
  =-\boldsymbol{\mathrm{ad}}^T_{\boldsymbol{V}^0} \boldsymbol{W}^0_{j}\,+  \boldsymbol{\mathrm{Ad}}_{\mathbf{C}^{-1}}^T \dot{\boldsymbol{W}}^b_{j}.
\end{equation}
Here, $\dot{\boldsymbol{W}}^b_{j}$ is the component-wise time derivative of $\boldsymbol{W}^b_{j}$. In general, the $r$-th time derivative is
{{\begin{equation}
\resizebox{\columnwidth}{!}{$
(\boldsymbol{W}^0)^{(r)}=
\begin{cases}
\boldsymbol{\mathrm{Ad}}_{\mathbf{C}^{-1}}^{T}\boldsymbol{W}^{b}, & r=0,\\[4pt]
\displaystyle
 \sum_{k=0}^{r-1} \binom{r-1}{k} \Big(-
\boldsymbol{\mathrm{ad}}^T_{(\boldsymbol{V}^0)^{(k)}} (\boldsymbol{W}^0)^{(r-k-1)}
+ (\boldsymbol{\mathrm{Ad}}_{\mathbf{C}^{-1}}^T)^{(k)} (\boldsymbol{W}^b)^{(r-k)}
\Big), & r\ge 1.
\end{cases}
$}
\label{extweriv}
\end{equation}}}
where the $r$-th time derivative of $\boldsymbol{\mathrm{Ad}}_{\mathbf{C}^{-1}}^T$ is given by applying the chain rule to \eqref{Adjointderiv}, which yields
\begin{equation}
     (\boldsymbol{\mathrm{Ad}}_{\mathbf{C}^{-1}}^T)^{(r)} =- \sum_{k=0}^{r-1} \binom{r-1}{k} \boldsymbol{\mathrm{ad}}^T_{(\boldsymbol{V}^{0})^{(k)}} (\boldsymbol{\mathrm{Ad}}_{\mathbf{C}^{-1}}^T)^{(r-k-1)}.
\end{equation}


\end{document}